\documentclass[journal,letterpaper]{IEEEtran}
\usepackage{amsmath,amsfonts}
\usepackage{algorithm}
\usepackage{algorithmic}
\usepackage{array}
\usepackage{textcomp}
\usepackage{stfloats}
\usepackage{url}
\usepackage{svg}
\usepackage{verbatim}
\usepackage{graphicx}
\usepackage{hyperref}
\def\BibTeX{{\rm B\kern-.05em{\sc i\kern-.025em b}\kern-.08em
    T\kern-.1667em\lower.7ex\hbox{E}\kern-.125emX}}
\usepackage{balance}
\usepackage{xcolor}
\usepackage{siunitx}
\usepackage{multirow}
\usepackage{tikz}
\usepackage{amsthm}
\usepackage{cleveref}

\newtheorem{theorem}{Theorem}

\newtheorem{lemma}[theorem]{Lemma}

\newtheorem{definition}{Definition}
\newtheorem{assumption}{Assumption}

\usepackage{thm-restate}

\newcommand{\bN}{\mathbb{N}}
\newcommand{\bR}{\mathbb{R}}

\newcommand{\cA}{\mathcal{A}}

\newcommand{\cM}{\mathcal{M}}

\newcommand{\cR}{\mathcal{R}}
\newcommand{\cS}{\mathcal{S}}
\newcommand{\cT}{\mathcal{T}}

\newcommand{\fpi}{\boldsymbol\pi}

\newcommand{\jhedit}[1]{{\color{black}#1}}

\newcommand{\jhdelete}[1]{}
\newcommand{\jhhdelete}[1]{{\color{lightgray}#1}}
\renewcommand{\jhedit}[1]{#1}

\renewcommand{\jhdelete}[1]{}
\renewcommand{\jhhdelete}[1]{}

\ifCLASSOPTIONcompsoc
    \usepackage[caption=false, font=normalsize, labelfont=sf, textfont=sf]{subfig}
\else
    \usepackage[caption=false, font=footnotesize]{subfig}
\fi

\begin{document}
\title{
    {Temporal Transfer Learning} for {Traffic Optimization} with {Coarse-Grained} {Advisory Autonomy}
}
\author{Jung-Hoon Cho, 
        Sirui Li,
        Jeongyun Kim, 
        Cathy Wu
        \thanks{Received 24 February 2025; revised 28 August 2025; accepted 7 October 2025. Date of publication 24 November 2025; date of current version 8 December 2025. This work was supported in part by the National Science Foundation (NSF) CAREER under Award \#2239566, in part by the Kwanjeong Educational Foundation Ph.D. Scholarship Program, in part by the MIT Energy Initiative (MITEI) Mobility Systems Center, the Kwanjeong scholarship, in part by the NSF under Award 2149548, and in part by the MIT Amazon Science Hub. This article was recommended for publication by Associate Editor J. Le Ny and Editor J. Bohg upon evaluation of the reviewers’ comments. (Corresponding author: Jung-Hoon Cho.)}
        \thanks{Jung-Hoon Cho is with the Department of Civil and Environmental Engineering and the Laboratory for Information \& Decision Systems, Massachusetts Institute of Technology, Cambridge, MA 02139, USA. (e-mail: jhooncho@mit.edu)}
        \thanks{Sirui Li is with the Institute for Data, Systems, and Society and the Laboratory for Information \& Decision Systems, Massachusetts Institute of Technology, Cambridge, MA 02139, USA. (e-mail: siruil@mit.edu)}
        \thanks{Jeongyun Kim is with the Department of Mechanical and Automotive Engineering, Seoul National University of Science and Technology, Seoul, South Korea (e-mail: jkim@seoultech.ac.kr)}
        \thanks{Cathy Wu is with the Laboratory for Information \& Decision Systems; the Institute for Data, Systems, and Society; and the Department of Civil and Environmental Engineering, Massachusetts Institute of Technology, Cambridge, MA 02139, USA. (e-mail: cathywu@mit.edu)}
    }

\markboth{}%
{Cho \MakeLowercase{\textit{et al.}}: Temporal Transfer Learning for Traffic Optimization with Coarse-Grained Advisory Autonomy}


\maketitle

\begin{abstract}
    The recent development of connected and automated vehicle (CAV) technologies has spurred investigations to optimize dense urban traffic\jhedit{, maximizing vehicle speed and throughput}. This paper explores \textit{advisory autonomy}, in which real-time driving advisories are issued to \jhedit{human} drivers, thus \jhedit{achieving near-term performance of automated vehicles}. Due to the complexity of traffic systems, recent studies of coordinating CAVs \jhedit{have leveraged} deep reinforcement learning (RL). Coarse-grained advisory is formalized as zero-order holds, and we consider a range of hold durations from 0.1 to 40 seconds. However, despite the similarity of the higher-frequency tasks for CAVs, a direct application of deep RL fails to generalize to advisory autonomy tasks. 
    \jhedit{To overcome this, we employ zero-shot transfer, training policies on a set of source tasks---specific traffic scenarios with designated hold durations---and then evaluating the efficacy of these policies on different target tasks.}
    We introduce \textit{Temporal Transfer Learning} (TTL) algorithms to select source tasks \jhedit{for zero-shot transfer}, systematically leveraging the temporal structure to solve the full range of tasks. TTL selects the most suitable source tasks to maximize the performance of the range of tasks. We validate our algorithms on diverse mixed-traffic scenarios, demonstrating that TTL more reliably solves the tasks than baselines. This paper underscores the potential of coarse-grained advisory autonomy with TTL in traffic flow optimization.
\end{abstract}


\section{Introduction}

\IEEEPARstart{R}{ecent} advancements in connected and automated vehicle (CAV) technologies have opened up new frontiers in addressing the challenges of urban traffic congestion and associated environmental problems. 
The growing urgency to mitigate traffic-related issues, buoyed by advances in autonomous vehicles (AVs) and machine learning, is pushing the boundaries of urban roadway autonomy. As the transportation sector progressively moves towards a fully autonomous paradigm, the spotlight is firmly on devising innovative methods for traffic flow optimization, targeting key outcomes such as enhanced eco-driving, throughput maximization, and congestion reduction \cite{wu_emergent_2017, stern_dissipation_2018}.

\begin{figure}[!t]
\centering
\includegraphics[width=3.4in]{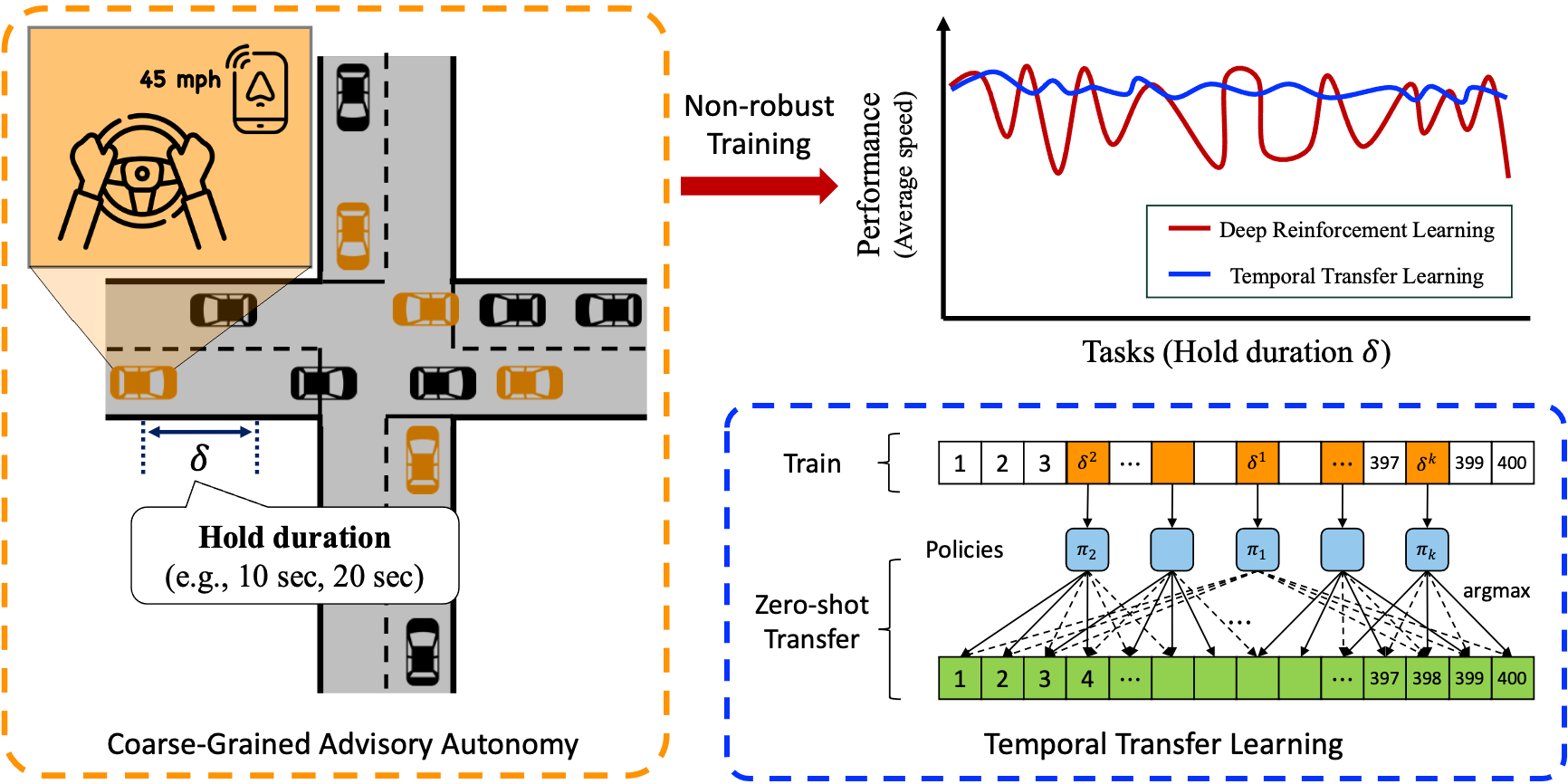}
\caption{\textbf{Illustrative figure of Temporal Transfer Learning (TTL) for the coarse-grained advisory system.} In a \textit{coarse-grained advisory system}, vehicles receive persistent guidance for a specified hold duration rather than instantaneous controls. The system performance of this system shows the non-robustness to the hold duration of deep reinforcement learning when trained exhaustively. We propose \textit{Temporal Transfer Learning (TTL)} methods designed to select source training tasks based on temporal features. In comparison to the exhaustive and multi-task training methods, TTL provides an intermediate number of policies to train to solve a full set of tasks.}
\label{fig:concept}
\end{figure}

This paper highlights the significant role of \textit{advisory autonomy}, \jhedit{an approach where automated systems provide real-time driving guidance to human drivers} to integrate seamlessly with other traffic and achieve better traffic flow. 
\jhedit{The crux of our research lies in demonstrating how advisory autonomy can enable human-driven vehicles to emulate the system-level performance of automated vehicles (AVs), providing a viable, cost-effective alternative in the near term.}
The notion of \textit{coarse-grained advisory autonomy} is formalized through the lens of coarse-grained zero-order holds \cite{sridhar_piecewise_2021}.
With this coarse-grained advice, instead of instantaneous controls (\cite{yan_reinforcement_2021, wu_flow_2022, yan_unified_2022}), vehicles are provided with guidance that persists for a particular duration, thereby addressing the intricacies of fluctuating hold duration.
This is significant as human drivers, unlike AVs, may find it challenging to adhere to frequent and rapid control changes.
\jhedit{Concurrently, work on robustness to human compliance errors (response delay, speed deviation) shows that advisory performance can degrade sharply without explicit handling \cite{kim2025reinforcement}.}

\jhedit{In this work,} our objective is to develop an algorithm that, given a traffic scenario, can determine whether guidance that human drivers could conceivably follow would achieve outcomes comparable to those of AVs.
We concentrate on human compatibility for traffic optimization and the ability of human drivers to match corresponding system-level metrics (such as the average speed of all vehicles and the throughput) rather than achieve accurate maneuvers where AVs possess clear advantages, such as being able to react to abrupt braking without hesitation.
This article subsumes and extends Sridhar \textit{et al.} \cite{sridhar_piecewise_2021}, which originally formulated the problem as piecewise-constant control for traffic optimization. \jhedit{This paper further elaborates to include both acceleration and speed guidance and validates on different traffic networks.}

Integrating this approach with reinforcement learning (RL) presents an elegant way forward, given RL's structured framework for sequential decision-making. 
\jhedit{While deep RL has emerged as a potent tool for this purpose, its direct application to the advisory system has exposed a degree of \jhedit{instability}, characterized by jagged performance over a range of tasks, echoing the findings of Sridhar \textit{et al.} \cite{sridhar_piecewise_2021}. This inconsistency and brittleness necessitate a more sophisticated approach to deep RL, one that can more reliably handle the complexities of real-world traffic scenarios encountered by advisory systems.}

To confront these challenges head-on, we turned to transfer learning, a widely employed technique in numerous research fields that enables the utilization of knowledge acquired from one task to enhance performance in another related task \cite{taylor_transfer_2009, pan_survey_2010}.
Specifically, transfer learning can be applied to adapt the pre-trained policy to a new task or to initialize a learning algorithm with pre-existing knowledge, substantially expediting the learning process and boosting overall performance.
Transfer learning has been successfully applied to improve the efficiency and training performance of traffic management systems \cite{kreidieh_dissipating_2018, jang_simulation_2019, yan_reinforcement_2021}.
\jhedit{In this context, we employ zero-shot transfer, where policies are trained on a source task--specific traffic scenarios with designated hold duration--and then evaluated on different target tasks. This approach is computationally efficient as it obviates the need for any additional fine-tuning, directly leveraging the trained policies to new scenarios.}

We introduce two \textit{Temporal Transfer Learning} (TTL) algorithms-- Greedy Temporal Transfer Learning (GTTL) and Coarse-to-fine Temporal Transfer Learning (CTTL). These algorithms adeptly leverage the temporal similarities across tasks to judiciously select training tasks, thereby significantly facilitating the training efficiency and overall performance. The essence of TTL lies in its capability to seamlessly transfer knowledge acquired from one task to another, circumventing the often observed training brittleness in deep RL algorithms.
\jhedit{The TTL approach provides a structured advantage by systematically leveraging temporal structures inherent in the task domain. Our GTTL methods especially leverage a linear generalization gap\jhedit{, which allows better estimation of zero-shot generalization across different tasks.}}
This ability of TTL to draw insights from prior models offers a promising avenue to circumvent the fragility often observed in deep RL training.
Then, to evaluate our algorithm's generalizability, we consider validation on various traffic scenarios in which mixed-autonomy traffic has been proven effective for traffic optimization \cite{wu_flow_2022, yan_unified_2022}.

The core contributions of this paper are twofold:
\begin{itemize}
    \item We delve into a \textit{coarse-grained advisory}, presenting a compelling case for its viability in enhancing system-level traffic outcomes. Our empirical evidence underscores the possibility of furnishing human drivers with guidance that mirrors AV behavior, leading to tangible traffic improvements. Such findings pave the way for considering human drivers as immediate, practical alternatives to full-fledged AV deployments.
    \item Our research introduces \textit{Temporal Transfer Learning (TTL)} algorithms, a robust methodology specifically designed to tackle the training brittleness intrinsic to deep RL algorithms. TTL can be promising in evolving generalizable training paradigms for complex traffic optimization tasks by adeptly identifying sources of variation and harnessing insights from pre-existing models.
\end{itemize}

\section{Related Work}
\subsection{\jhedit{Reinforcement Learning for Mixed Autonomy Traffic}}
As we await the era of fully automated vehicles, we can anticipate a mixed autonomy system where automated and human-driven vehicles share the road. 
In such a system, controlling only a small proportion of the vehicles can significantly improve the overall traffic flow \cite{wu_flow_2022, kreidieh_dissipating_2018}. 
Several studies have explored the potential of RL in addressing the challenges posed by the coexistence of AVs and human-driven vehicles. 
Researchers have worked on enhancing traffic efficiency in mixed autonomy settings using deep RL-based approaches, showing that it can eliminate stop-and-go traffic and mitigate congestion \cite{wu_flow_2022, kreidieh_dissipating_2018, stern_dissipation_2018, vinitsky_lagrangian_2018, yan_reinforcement_2021, yan_unified_2022}. 
These studies collectively highlight the potential of RL in optimizing mixed autonomy traffic, paving the way for enhanced safety, efficiency, and performance in transportation systems.
\vspace{-0.2em}
\subsection{Advisory Autonomy}
Advisory systems in roadway autonomy span a broad range of applications, from enhancing safety to mitigating traffic congestion. 
These systems provide considerable benefits to users. 
For instance, collision warning alerts have been employed to ensure the driver’s safety \cite{bishop_intelligent_2000}, and speed advisory systems at signalized intersections help users pass the green light efficiently \cite{katsaros_performance_2011}. 
At a system level, on the other hand, the advisory system provides system-level traffic optimization. 
For example, speed advisory systems contribute significantly towards eco-driving \cite{xiang_closed-loop_2015} and personalized advisory systems have been introduced to mitigate traffic congestion \cite{hasan_perp_2023, hasan2024cooperative}. 
Furthermore, roadway signs suggesting advisory speeds represent another form of advisory autonomy.

However, the application of advisory systems poses unique challenges given their interaction with human drivers. While fully automated vehicles can operate within clearly defined parameters and constraints, human drivers exhibit different behaviors. For instance, as noted by Mok \textit{et al.}, humans require a minimum of 5-8 seconds to appropriately transition control \cite{mok_emergency_2015}. \jhedit{This finding highlights the importance of considering the distinct characteristics and limitations of human drivers when developing control methods.} For instance, Sridhar \textit{et al.} identified two key characteristics of human-compatible driving policies: a simple action space and the capacity to maintain the same action for a few seconds \cite{sridhar_piecewise_2021}.
An example of a human-compatible advisory system is a coarse-grained control system, which is provably stable in the context of Lyapunov in mitigating congestion on single-lane ring roads \cite{li2023integrated}. This system, known as an action-persistent Markov Decision Process (MDP), successfully addresses the human need for simplicity and persistent actions.
In light of these considerations, it is crucial to integrate human driving characteristics into the design of control methods for human drivers. 
\jhedit{Complementary to our effort in developing the transfer learning based algorithm, Kim et al.\ \cite{kim2025reinforcement} relax the perfect-compliance assumption by modeling driver delays and speed deviations, revealing substantial degradation without robustness measures and proposing RL-based advisories resilient to such errors.}

\subsection{\jhedit{Action Persistent MDPs and RL} \label{sec:action-repetition}}
\jhedit{In exploring action repetition within reinforcement learning, the concept of Semi-Markov Decision Processes (Semi-MDPs) offers a rich framework for incorporating temporally abstract actions into the traditional MDP paradigm \cite{sutton_between_1999}. The ability to apply the same action across extended time periods allows for a simplified control strategy that can be beneficial for complex control problems.}

Various methodologies such as reducing control granularity, implementing a skip policy, and applying temporal abstraction have been employed to analyze action repetition \cite{lakshminarayanan_dynamic_2017, sharma_learning_2020, biedenkapp_temporl_2021, metelli_control_2020}.
Metelli \textit{et al.} introduced action persistence and the Persistent Fitted Q-Iteration (PFQI) algorithm to modify control frequencies and learn optimal value functions \cite{metelli_control_2020}. 
Lee \textit{et al.} addressed the multiple control frequency problems that guarantee convergence to an optimal solution and outperform baselines \cite{lee_reinforcement_2020}. 
\jhedit{In the context of transportation systems, Sridhar and Wu investigate the use of piecewise constant policies for traffic congestion mitigation, providing a structured approach to guide human drivers in real-time \cite{sridhar_piecewise_2021}.}

\jhedit{These contributions collectively underscore the significance of action repetition and the strategic choice of control frequencies in RL. They also highlight the potential for translating these concepts into tangible traffic management solutions, exemplifying the intersection between theoretical research and practical application.}

\subsection{\jhedit{Transfer Learning in RL}}
Transfer learning is a popular technique used in various research domains to leverage the knowledge gained from one task to improve performance in another related task \cite{taylor_transfer_2009, pan_survey_2010}. 
In particular, transfer learning can be used to adapt a pre-trained policy to a new task or to initialize a learning algorithm with pre-existing knowledge\jhedit{, thereby greatly accelerating the learning process and improving overall performance.}
In contrast to multitask learning's simultaneous approach, transfer learning applies knowledge from source tasks to optimize a particular target task, underscoring an asymmetrical relationship between tasks \cite{pan_survey_2010}.

Transfer learning offers the advantage of significantly decreasing the amount of data needed for learning compared to traditional independent learning methods \cite{yang_theory_2013}. 
Dynamic transfer learning maps for multi-robot systems can be obtained \jhedit{from} the basic system properties \jhedit{of} approximated physical models or experiments \cite{helwa_multi-robot_2017}. 
Kouw and Loog not only delved into the specific instances and various techniques of domain adaptation but also highlighted the challenges of sequential domain adaptation \cite{kouw_introduction_2019}. 
Moreover, transfer learning also has its benefits, \jhedit{as it requires a reduced number of data or training for new tasks, stemming from the shared representation of related tasks} \cite{yang_theory_2013, tripuraneni_theory_2020, guan_task_2022}.

In robotics, transfer learning has been utilized for a wide range of applications such as robot manipulation, locomotion, and control \cite{higgins_darla_2018, rusu_sim--real_2018}. 
In the context of traffic settings, transfer learning has been applied to improve the efficiency and training performance of traffic management systems \cite{kreidieh_dissipating_2018, jang_simulation_2019, yan_reinforcement_2021}. 
For example, Kreidieh \textit{et al.} proposed a transfer learning framework that can help the warm start for training policies to dissipate shockwaves from closed traffic scenarios to more complex open ones \cite{kreidieh_dissipating_2018}. 
Similarly, zero-shot policy transfer to adapt a pre-trained policy for autonomous driving in a structured environment to an unstructured environment results in improved performance and safety \cite{jang_simulation_2019}. 

Also, the transferability of the learned policies may differ at different levels of tasks; for instance, policies derived from more structured and informative tasks are more robust to diverse tasks \cite{kreidieh_inter-level_2021}. 
Yan \textit{et al.} proposed a unified framework for traffic signal control using transfer learning to transfer knowledge across different intersections and adapt to varying traffic conditions \cite{yan_reinforcement_2021}. 
Also, transfer learning is used for real-time crash prediction \cite{man_transfer_2022}, and traffic flow prediction in data-sparse regions \cite{oruche_transfer_2021}.

RL-based methods require generating significant amounts of simulation data, which can be costly. 
However, transfer learning offers a solution to alleviate the burden of data generation and simulation for training each model. 
By employing an efficient training scheme, the model can quickly learn when, what, and where to transfer knowledge in scenarios with limited data availability \cite{jang_learning_2019}.
\jhedit{The selection of source tasks is critical in transfer learning as it sets the foundation for the efficacy of knowledge transfer. Contextual relevance in source task selection is critical for the efficacy of the transferability of a policy, which may be predicted through its relation to the target task's characteristics \cite{sinapov_learning_2015, li_context-aware_2019}. Furthermore, Agostinelli \textit{et al.} explore metrics that predict the success of transferred knowledge, facilitating the selection of source model ensembles to maximize performance on the target task \cite{agostinelli_transferability_2022}.
In addition to selecting individual source tasks, multi-task learning can also be used to solve multiple related tasks \cite{belletti_expert_2018}.
}
\jhedit{Closest to our setting, Cho et al.\ introduce Model-Based Transfer Learning (MBTL) \cite{cho2024model}, which explicitly models (i) the training performance via Gaussian processes and (ii) a linear generalization gap over context, and then uses Bayesian optimization to pick source tasks with sublinear regret guarantees. In contrast, our TTL specializes in the temporal context (hold duration~$\delta$) and yields simpler closed-form greedy and coarse-to-fine selection rules with area-coverage guarantees.}

A hierarchical approach to task granularity can be beneficial as it allows for the refinement of coarse attributes while learning finer tasks. 
This method has been successfully employed by Wei \textit{et al.} in their work on vehicle re-identification tasks \cite{wei_coarse--fine_2018} and in large-scale fault diagnosis tasks \cite{wang_coarse--fine_2023}. 
\jhedit{Our method leverages temporal locality in hold duration---transferring between nearby $\delta$---which plays an analogous role to curriculum learning while explicitly optimizing source-task selection under a fixed training budget.}

Overall, transfer learning has shown promising results in improving the efficiency and safety of traffic management systems by leveraging the similar temporal structure of a series of tasks and prior knowledge from related tasks.

\section{\jhedit{Coarse-Grained Control}}
\subsection{\jhedit{Coarse-Grained Guidance in Advisory Autonomy}}
Advisory autonomy stands for the automated system that provides guidance to human drivers rather than a fully controllable process.
In this context, it is designed to work in the presence of human-driven vehicles, ensuring that controlled vehicles operate in a manner that is safe, predictable, and intuitive for human drivers. 
\textit{Coarse-grained control} refers to the vehicle control system that gives control periodically. 
Coarse-grained control involves applying the same action to an autonomous vehicle for a fixed time segment.
As we discussed in \Cref{sec:action-repetition}, coarse-grained control can be interpreted as action-persistent MDPs with different control granularities.

\subsection{\jhedit{Action Persistent MDPs}}

\jhedit{\jhedit{We consider $N$ vehicles} and assume that all vehicles are human-driven vehicles. A subset of these vehicles, defined by the fraction $\rho$, receives periodic guidance from the advisory system and is termed \textit{guided vehicles}. The remaining vehicles, constituting the fraction $(1-\rho)$, are designated as \textit{default-driven vehicles} and do not receive such guidance.}
Guided vehicles are human-driven vehicles with periodic assistance from a trained policy for coarse-grained control, $\pi(s_{t_m})$.
The policy is applied at intervals $t_m=\delta m$, where $m \in \bN_0$ ($\bN_0$ as a set of non-negative integers) and $\delta$ denotes the guidance hold duration. \jhedit{It's essential that the guidance hold duration is significantly shorter than the total horizon, denoted as ($H\gg\delta$).}
These vehicles receive guidance for any time $t$ that falls within the range $[t_m, t_{m+1}]$. 
This action persistent MDP can be represented by the 7-tuple $\cM_\delta=(\cS,\cA,\cT,\cR,H,\gamma,\delta)$\jhedit{, where $\cS$ defines the state space, $\cA$ is action space, $\cT:\cS\times\cA\times\cS\rightarrow\bR$ is a transition probability distribution, $\cR$ is reward function, $H$ is a total time horizon, and $\gamma$ is a discount factor. 
The transition probability function $P(s'|s,a)$ specifies the probability of transitioning to a state $s'$ from a state $s$ by taking action $a$. 
An agent's objective in an MDP is to find a policy $\pi$ that maximizes the expected sum of rewards obtained over time, given the current state $s$ and the actions it can take.} 

\jhedit{\textit{Coarse-grained control}, also known as piecewise constant control or zero-order hold control, refers to the application of the same control action over a specified time segment length \cite{sridhar_piecewise_2021}.}
In other words, the same action \jhedit{determined at time step $t_m$} is applied to the time segment of $t\in[t_m,t_{m+1}]$.
In the single-lane ring, the simulation experiments reported that the hold duration could be extended to 24 seconds without degradation of the system performance \cite{sridhar_piecewise_2021}.
This piecewise constant control is backed up with the simulator experiments to evaluate the effect of the coarse-grained advisory \cite{hasan_towards_2023}.
Hasan \textit{et al.} also introduces a cooperative advisory system that leverages a novel driver trait conditioned Personalized Residual Policy (PeRP) to guide drivers in ways that reduce traffic congestion \cite{hasan_perp_2023, hasan2024cooperative}.

\subsection{\jhedit{Guidance Type}}
\jhedit{In our setting, we assume that the guided vehicle $i$ observes the space headway $s_i(t)$ between $i$th vehicle and the preceding vehicle and its derivative $\dot{s}_i(t)$ can be computed from successive headway measurements or velocity differences when the preceding vehicle’s velocity can be observed. Likewise, the velocity $v_i(t)$ comes from speed sensors in the vehicle.}
\jhedit{Based on the observation of the drivers, \jhedit{the coarse-grained controller} provide\jhedit{s} the drivers with either the target raw acceleration or target speed.}

\textbf{\jhedit{Acceleration Guidance. }}
\jhedit{Acceleration guidance in advisory autonomy systems directs human drivers by recommending the optimal acceleration action from a continuous action set, determined by the trained policy $\pi$. Concretely, the policy $\pi$ takes the observed traffic state at time $t_m$---namely $\bigl(s_i(t_m), \dot{s}_i(t_m), v_i(t_m)\bigr)$---and outputs a recommended acceleration for guided vehicle $i$ for the duration $t \in [t_m,\,t_{m+1} = t_m + \delta]$.}

For this acceleration guidance in the single-lane ring, Lyapunov analysis gives sufficient conditions for the stability of the coarse-grained advisory \cite{li2023integrated}.
\jhedit{Moreover, the success of this guidance system critically depends on the interface design through which advisories are communicated. A few challenges are possible discomfort for drivers, the complexity of human drivers in accurately interpreting and executing precise acceleration commands, and an increased risk of manual execution errors.}

\textbf{\jhedit{Speed Guidance. }}
\jhedit{In contrast to acceleration guidance, speed guidance presents human drivers with a target speed that they should attain and maintain over the hold duration interval. This approach is motivated by the observation that many drivers find it more intuitive to adjust to a specific speed target rather than precisely following recommended accelerations \cite{macadam_understanding_2003}.}
Moreover, using acceleration guidance type for the coarse-grained control often struggles to achieve and sustain the optimal velocity as discussed in \cite{bando_dynamical_1995, sugiyama_optimal_1999}.
\jhedit{Accordingly, our trained policy $\pi$ generates a recommended target speed for guided vehicle $i$ based on its current state $(s_i(t_m), \dot{s}_i(t_m), v_i(t_m))$. Once this target speed is communicated, the driver accelerates or decelerates to reach it as quickly and smoothly as possible, subject to comfort and safety constraints.}

Researchers have worked on the speed advisory system \cite{liang_joint_2019, liu_simulation_2019, ma_vehicle_2021}. 
For example, Liang \textit{et al.} guided the driver with the speed for signal phase and timing in CAV environment \cite{liang_joint_2019}.
Wang \textit{et al.} reported that the human-machine interface displaying the difference between current and suggested speeds with the cooperative driving simulator improved the performance while displaying time difference harmed the speed adaptation \cite{wang_investigating_2023}.
However, there are some drawbacks that human drivers tend to perceive the target speed as the easily broken speed limit and easily exceed it \cite{mannering_empirical_2009}.

\Cref{fig:two-guide-types} intuitively depicts two distinct forms of \jhedit{advisory} provided to drivers: acceleration and speed guidance.
From a stabilization standpoint, speed guidance has certain advantages, as it enables the vehicle to maintain a constant speed throughout the hold duration. 
However, under acceleration guidance, the vehicle's speed is subject to change unless the acceleration is precisely zero. 
This inherent difference between the two forms of guidance gives rise to distinct behaviors and responses within the traffic system, as demonstrated in our results.

\begin{figure}[!t]
    \centering
    \subfloat[Acceleration guidance\label{fig:guide-accel}]{
        \includegraphics[width=1.67in]{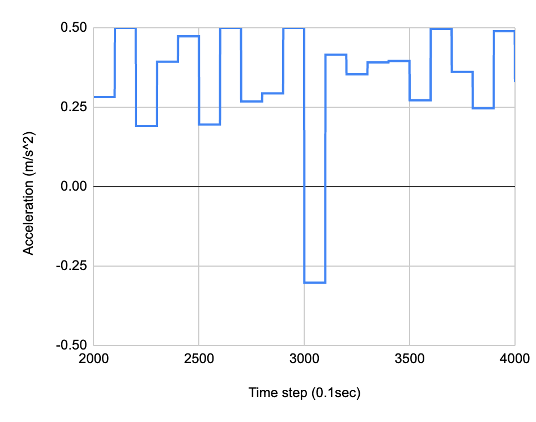}}
    \hfill
    \subfloat[Speed guidance\label{fig:guide-speed}]{
        \includegraphics[width=1.67in]{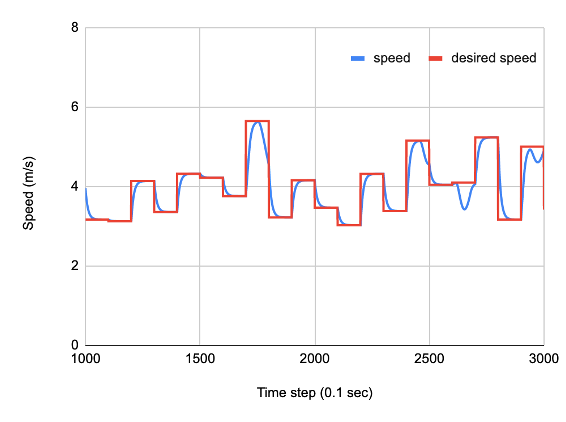}}
    \caption{Two types of advisory system to the human drivers: acceleration guidance (\ref{fig:guide-accel}), speed guidance (\ref{fig:guide-speed}).
    }
    \label{fig:two-guide-types}
\end{figure}

\section{Temporal Transfer Learning}
In advisory autonomy, we guide human drivers by providing a predetermined period, known as the hold duration, indicating how long they should maintain their guided actions. 
We consider solving families of MDP tasks whose only difference is the guidance hold duration, since the control of humans can vary. 
\jhedit{Throughout, we use “task” in the transfer-learning sense to denote the same control problem evaluated at a particular guidance hold duration~$\delta$. Apart from $\delta$, all environment and model parameters, such as the number of agents and road networks, are identical.}
Even in this setting, we find that RL will train successfully in some scenarios and unsuccessfully in others, with no clear pattern among the tasks. 
Similar findings have been documented in \cite{jayawardana_impact_2022}.

The algorithm introduced in this section is inspired by the intuition that an optimal strategy for hold duration $\delta$ should not differ significantly from that for hold duration $\delta' \approx \delta$. 
\jhedit{In particular, small perturbations of the hold duration should yield only minor changes to the action sequence and the closed-loop dynamics. Empirically, however, $J^*(\delta)$ and the learned policies exhibit a jagged, non-smooth dependence on $\delta$---a sign of training brittleness---which motivates zero-shot transfer across nearby tasks.}

\jhedit{We aim to investigate multiple tasks to comprehend the intricacies of the coarse-grained advisory and its effectiveness in optimizing traffic flow, particularly in mitigating congestion. Addressing multiple tasks provides a holistic understanding of the system's behavior under various scenarios, thereby informing more robust optimization strategies.}
While solving multiple tasks simultaneously with a separate model per each task could be computationally intensive and resource-demanding, leveraging a pre-trained model and transfer learning for our specific tasks can drastically reduce the computational burden. 
\jhedit{The list of variables and functions used throughout this paper is summarized in the \Cref{tab:notations}, ensuring clarity and ease of reference for the reader.}
\jhedit{
\begin{table}[!t]
    \caption{Table of Notations}
    \label{tab:notations}
    \centering
    \begin{tabular}{cl}
        \hline
        \textbf{Symbol} & \textbf{Description} \\
        \hline
        $\delta$ & Guidance hold duration\\
        $J(\delta)$ & The performance of task with duration $\delta$\\
        $A$ & The aggregate performance across different durations\\
        $k$ & Sequential step in source task selection\\
        $\delta^k$ & The guidance hold duration chosen at $k$th step\\
        $\pi_k$ & The policy trained at the task at $\delta^k$\\
        $J^{\pi_k}(\delta^k)$ & The performance of policy $\pi_k$ evaluated at $\delta^k$\\
        $\Delta J(\delta_S, \delta_T)$ & The generalization gap when the policy trained with $\delta_S$ \\
        &transferred to the task with $\delta_T$\\
        $J_k(\delta)$ & The performance updated with the best-performing\\
        &performance among previously trained policies\\
        $S_k$ & A set of selected source tasks\\
        \hline
    \end{tabular}
\end{table}
}

\subsection{\jhedit{Problem Definitions}}

\textbf{\jhedit{Guidance Hold Duration and Performance. }}
\jhedit{In coarse-grained advisory settings, let $J(\delta)$ denote the performance metric for a task with a guidance hold duration of $\delta$. For a coarse-grained policy $\pi$, we write this as $J^\pi(\delta)$. The hold duration $\delta$ spans from $\delta_\text{min}$ to $\delta_\text{max}$.}
\jhedit{We define the \emph{aggregate performance} over the interval $[\delta_\text{min}, \delta_\text{max}]$ as the integral of $J(\delta)$:}
\begin{equation}
    \jhedit{A(\delta_\text{min}, \delta_\text{max}) \;=\; \int_{\delta_\text{min}}^{\delta_\text{max}} J(\delta) \,\mathrm{d}\delta,}
\end{equation}
\jhedit{which measures how the system performs across all hold durations. In practice, we approximate this integral via a discrete sum:}
\begin{equation}
    \jhedit{\tilde{A}(\delta_\text{min}, \delta_\text{max}) \;=\; \sum_{\delta=\delta_\text{min}}^{\delta_\text{max}} J(\delta).}
\end{equation}
\jhedit{For notational simplicity, we will use $A(\delta_\text{min}, \delta_\text{max})$ or simply $A$ to represent this aggregate performance.}

\textbf{\jhedit{Sequential source tasks selection problem. }}
\jhedit{We next define the \emph{sequential source tasks selection problem}, where the goal is to iteratively choose which hold duration $\delta^k$ to train on so as to maximize overall performance. We let $K$ denote the transfer budget (i.e., the maximum number of source tasks we can train). At each iteration $k \in \{1,\dots,K\}$, we train a policy on the task with hold duration $\delta^k$. The set of selected source tasks at iteration $k$ is denoted by $S_k$.}

\jhedit{We write $J_k(\delta)$ for the estimated performance on a task with hold duration $\delta$ after $k$ iterations of training on selected tasks. Training at $\delta^k$ produces a policy $\pi_k$ with performance $J^{\pi_k}(\delta^k)$ on that same task. In practice, a policy trained on one task $\delta_S$ may perform suboptimally when applied to a different task $\delta_T$, due to a \emph{generalization gap}. This gap captures how performance degrades when zero-shot transferring a policy from $\delta_S$ to $\delta_T$.} \jhedit{Such \jhedit{performance degradation} has been studied in literature \cite{wang_generalization_2019, benjamins_contextualize_2023}, and similar effects have been observed in multi-objective contexts \cite{garau-luis_multi-objective_2022}. The generalization gap is crucial for understanding the limits of transfer and for guiding the iterative improvement of task-specific policies in reinforcement learning.}

\begin{definition}[Generalization Gap \jhedit{$\Delta J(\delta_S, \delta_T)$}]
    For a policy trained on source task $\delta_S$ and evaluated on target task $\delta_T$, \jhedit{we define the \emph{generalization gap} as the difference in performance $\Delta J(\delta_S,\delta_T)=J^{\pi_S}(\delta_S)-J^{\pi_S}(\delta_T)$.}
\end{definition}

\jhedit{We initialize the estimated performance $J_0(\delta)$ for all guidance hold durations $\delta$ within the range of interest to zero, which sets the baseline for subsequent improvements:}
\begin{equation}
    J_0(\delta)=0 \quad \forall \delta \in [\delta_\text{min}, \delta_\text{max}].
\end{equation}

\jhedit{After each iteration $k$, $J_k(\delta)$ is updated by taking the maximum over the best known policy so far and the newly trained policy (accounting for the generalization gap):}
\begin{equation}
    J_k(\delta) = 
    \begin{cases}
        J^{\pi_k}(\delta^k), & \text{if } \delta = \delta^k, \\
        \max \bigl(J_{k-1}(\delta),\; J^{\pi_k}(\delta^k) - \Delta J(\delta^k, \delta)\bigr), & \text{otherwise}.
    \end{cases}
\end{equation}

\begin{definition}[Sequential Source Tasks Selection Problem]
    \jhedit{The problem is to select a sequence of tasks $\{\delta^1,\delta^2,\dots,\delta^K\}$ that maximizes the aggregate performance. At iteration $k$, we pick $\delta^k$ to maximize $A_k(\delta^k)$, where}
    \begin{equation}
        \jhedit{
        A_k(\delta^k)
        = \int_{\delta_{\text{min}}}^{\delta_{\text{max}}} 
        \max\bigl(J_{k-1}(\delta),\; J^{\pi_k}(\delta^k)-\Delta J(\delta^k, \delta)\bigr)\,\mathrm{d}\delta.
        }
    \end{equation}
    \jhedit{The selection process continues for up to $K$ iterations, thereby incrementally improving policy performance across all hold durations.}
    \label{def:sequential-source-tasks-selection}
\end{definition}

\begin{figure}[!t]
    \centering
    \begin{tikzpicture}[domain=0:5,scale=0.9]
        \fill[black!10] (0,0) -- (0,3.1) -- (0.9,4) -- (4.9,0) -- cycle;
        \draw[very thin,color=gray] (-0.1,-0.1) grid (4.9,4.9);
        
        \draw[->] (-0.2,0) -- (5.2,0) node[below=0.1, yshift=-8pt, pos=0.5] {Target task $\delta$};
        \draw[->] (0,-0.2) -- (0,5.2) node[left] {$J$};

        \draw[color=black] (0,1.5) -- (2.5,4) -- (5,1.5);
        \draw[color=black] (0,3.1) -- (0.9,4) -- (4.9,0);
        \draw[black,thick] (0,3.1) -- (0.9,4) -- (1.7,3.2) -- (2.5,4) -- (5,1.5);
        \draw[black,dotted,thick] (0.9,4) -- (0,4) node[left] {$J^{\pi_1}(\delta^1)$};
        \draw[black,dotted,thick] (2.5,4) -- (5,4) node[right] {$J^{\pi_2}(\delta^2)$};
        \draw[black,dotted,thick] (2.5,0) node[below] {$\delta^2$} -- (2.5,4);
        \draw[black,dotted,thick] (0.9,0) node[below] {$\delta^1$} -- (0.9,4);
        \draw[black,dotted,thick] (3.7,0) node[below] {$\delta$} -- (3.7,2.8);
        \draw[<->,thick] (4.2,1.2) -- (4.2,4) node[pos=0.35, right=0.3] {$\Delta J(\delta^1,\delta)$};
        \draw[<->,thick] (4.5,2.8) -- (4.5,4) node[pos=0.5, right] {$\Delta J(\delta^2,\delta)$};
        
        \draw[black,dotted,thick] (3.7,2.8) -- (5,2.8) node[right] {$J_{2}(\delta)$};
        \draw[black,dotted,thick] (3.7,1.2) -- (5,1.2) node[right] {$J_{1}(\delta)$};
        \node[] at (1.5,1.5) {$A_1$};
    \end{tikzpicture}
    \caption{\jhedit{Visualization of sequential source task selection and corresponding performance evaluations within the guidance hold duration space. The shaded region represents the aggregate performance $A_1$ after selecting $\delta^1$ in the first step. The generalization gap $\Delta J(\delta^1,\delta)$ quantifies the performance drop when applying the policy trained at $\delta^1$ to a target task with $\delta$. At the second step, the selection of $\delta^2$ updates the estimated performance of task with duration of $\delta$ from $J_1(\delta)$ to $J_2(\delta)$.}}
    \label{fig:illust-for-notation}
\end{figure}
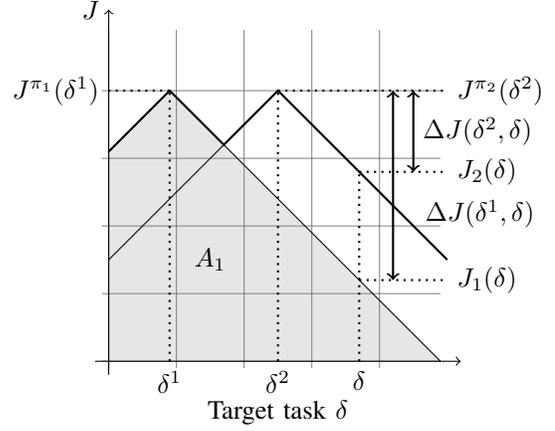

\jhedit{A conceptual illustration of this process is shown in \Cref{fig:illust-for-notation}. At each step, the newly trained policy may enhance performance in a neighborhood of $\delta^k$, subject to the generalization gap when evaluated on other tasks.}

\subsection{\jhedit{Modeling Assumptions} \label{sec:modeling-assumptions}}
\jhedit{\noindent We now list the main assumptions used to formalize and analyze the source tasks selection problem in the context of temporal transfer learning.}

\begin{assumption}[Constant {upper-bound} performance]
    \jhedit{For any hold duration $\delta$ in $[\delta_\text{min}, \delta_\text{max}]$, the upper-bound of the performance $J^*(\delta)$ is constant, denoted by $J^*$. Formally,}
    \begin{equation}
        \jhedit{J^*(\delta) = J^*, \quad \forall\, \delta \in [\delta_\text{min}, \delta_\text{max}].}
    \end{equation}
    \label{assume:constant-upperbound-j}
\end{assumption}
\jhedit{\Cref{assume:constant-upperbound-j} is supported by empirical analysis within coarse-grained advisory autonomy settings, suggesting that various coarse-grained guidance tasks may uphold the same \jhedit{training} performance. Our observations in single-lane ring environments validate this assumption (\Cref{fig:ring-scratch}), although it is noted that in more complex scenarios like highway ramps, the \jhedit{training} performance may decline as hold duration increases (\Cref{fig:ramp-scratch}).}

\begin{assumption}[\jhedit{Deterministic training performance}]
    \jhedit{For any task trained with hold duration $\delta^k$, \jhedit{training attains performance $J^{\pi_k}(\delta^k)\approx J^*(\delta^k)$; analysis uses $J^{\pi_{k}}(\delta^k)=J^*(\delta^k)$}.}
    \label{assume:successful-train}
\end{assumption}

\jhedit{From Assumptions \ref{assume:constant-upperbound-j} and \ref{assume:successful-train}, training on any chosen $\delta^k$ achieves $J^*$, which simplifies the analysis of subsequent transfers.}

\begin{assumption}[Linear generalization gap]
    \jhedit{The generalization gap $\Delta J(\delta_S, \delta_T)$ between tasks $\delta_S$ and $\delta_T$ is linearly proportional to $|\delta_S - \delta_T|$. Specifically,}
    \begin{equation}
        \Delta J(\delta_S, \delta_T) = 
        \begin{cases}
            \theta_L (\delta_S-\delta_T),& \text{if } \delta_S > \delta_T\\
            \theta_R (\delta_T-\delta_S), & \text{otherwise}
        \end{cases}
    \end{equation}
    where $\theta_L$ signifies the slope of transfer performance when transitioning from a coarser to a finer task, implying that $\delta_S > \delta_T$. Conversely, $\theta_R$ represents the slope when shifting from a finer to a coarser task, suggesting that $\delta_S < \delta_T$.
    \label{assume:linear-transfer}
\end{assumption}

\begin{assumption}[Symmetric generalization gap function]
    \jhedit{For simplicity, we assume the transfer slopes are equal, i.e., $\theta_L = \theta_R = \theta$.}
    \label{assume:same-transfer-slope}
\end{assumption}
\Cref{assume:same-transfer-slope} simplifies the analysis by asserting that the granularity of the task does not influence the rate of performance degradation during policy transfer.

\begin{assumption}[Bounded slope of generalization gap function]
    \jhedit{We require that $J^* \ge \theta(\delta_\text{max}-\delta_\text{min})$, so that}
    \begin{equation}
        \theta \leq \frac{J^*}{\delta_\text{max}-\delta_\text{min}}
    \end{equation}
    \label{assume:upperbound-J} 
\end{assumption}
\jhedit{This bound ensures the generalization gap does not exceed the maximum possible performance, making geometric analyses tractable.}
If $J^*$ is larger than $\theta(\delta_\text{max}-\delta_\text{min})$, the transfer from any point would be able to encompass the additional volume. 
Thus, without loss of generality, we can assume $J^*=\theta(\delta_\text{max}-\delta_\text{min})$ in our geometric analysis.
If $J^*$ is less than $\theta(\delta_\text{max}-\delta_\text{min})$, indicating a relatively constrained effective transfer range, the advantage of using transfer learning might be limited. 
This is due to the fact that TTL tends to benefit from the case where the amount of the generalization gap is prominent.

\subsection{\jhedit{Optimal Strategy for Source Tasks Selection Problem} \label{sec:optimal-incremental-transfer}}
With several assumptions we made in the previous section, we can devise a systematic algorithm to solve the source tasks selection problem and choose the subsequent training source task based on simple geometry.
We consider an analysis that simplifies the marginal performance improvement after each iteration to obtain intuition and provide a theoretical grounding for the TTL process.

\jhedit{As shown in Figure~\ref{fig:illust-for-notation} and supported by Assumptions~\ref{assume:constant-upperbound-j}--\ref{assume:linear-transfer}, training on a selected source task yields optimal performance at that task, creating two distinct segments in the performance function. These segments, delineated by the selected task, can be modeled as piecewise linear functions. After $(k-1)$ iterations, there are $k$ segments with inflection points at $\delta^1,\dots,\delta^{k-1}$. Our objective is to select the next hold duration $\delta^k$ that maximizes the aggregate performance $A_k$. To this end, we evaluate each piecewise linear segment and compute its potential marginal increase in $A_k$, which depends on the shape of $J(\delta)$.}

\jhedit{Figure~\ref{fig:illust-ttl} illustrates various decision rules for selecting the transfer point based on the shape of $J(\delta)$. Under Assumption~\ref{assume:constant-upperbound-j}, the performance upper-bound $J^*$ is depicted as a flat blue dotted line, while Assumptions~\ref{assume:linear-transfer} and \ref{assume:same-transfer-slope} ensure that the transfer performance functions in both directions are linear with the same slope.}

\begin{figure}[!t]
    \centering
    \includegraphics[width=2.8in]{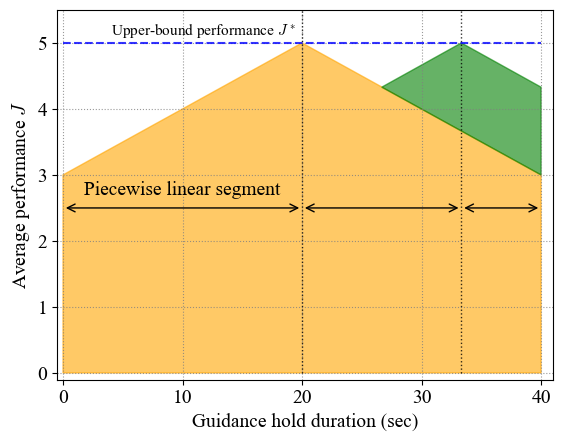}
    \caption{
    \jhedit{An exemplified representation of the Temporal Transfer Learning (TTL) process for source task selection. The graphic showcases the stepwise procedure for two iterations ($k=2$), resulting in two segments demarcated by inflection points at $\delta^1$ and $\delta^2$. The upper-bound performance $J^*$ is indicated by the blue dotted line, as posited in \cref{assume:constant-upperbound-j}, while the piecewise linear segments and their slopes, as governed by \cref{assume:linear-transfer} and \ref{assume:same-transfer-slope}, guide the selection of the next hold duration $\delta^k$ that will maximize the aggregate performance $A_k$. Each segment is assessed for its potential marginal contribution to $A_k$, with decisions influenced by the shape of the performance function $J(\delta)$, here visualized as transitions from the orange to the green area, signifying the shift in guidance hold duration from $\delta^1 = 20$ to $\delta^2 = 33.33$.}}
    \label{fig:illust-ttl}
\end{figure}

\jhedit{We formalize a greedy approach in \Cref{theorem:1-step-greedy-for-single-step}, which chooses $\delta^k$ to maximize the marginal increase in $A_k$ within each piecewise linear segment. Although ``greedy” focuses on one step at a time, it yields an efficient and interpretable selection rule under our linear-gap assumptions.}
\begin{theorem}[\jhedit{Optimal source task selection for greedy transfer}]
    \jhedit{Consider a piecewise linear segment $[\delta_L, \delta_R]$ of $J_k(\delta)$. To maximize the marginal increase in $A_k$, the greedy choice of $\delta^k$ is:}
    \begin{align}
        \delta^k &=\begin{cases}
            \frac{\delta_L+\delta_R}{2}&\text{for } k=1 \text{ or } J_k \text{ symmetric}\\
            \frac{2\delta_L+\delta_R}{3}&\text{for }k\neq1 \text{ and }\frac{\mathrm{d}J_k}{\mathrm{d}\delta}>0\\
            \frac{\delta_L+2\delta_R}{3}&\text{for }k\neq1 \text{ and }\frac{\mathrm{d}J_k}{\mathrm{d}\delta}<0
        \end{cases}
    \end{align}
    \jhedit{The resulting marginal increase $\Delta A_k$ in the aggregate performance is:}
    \begin{align}
        \Delta A_k=\begin{cases}
            \frac{3}{4}\theta(\delta_R-\delta_L)^2&\text{for }k=1\\
            \frac{1}{8}\theta (\delta_R-\delta_L)^2&\text{for }k\neq1 \text{ and }J_k \text{ symmetric}\\
            \frac{1}{3}\theta (\delta_R-\delta_L)^2&\text{otherwise}.
        \end{cases}
    \end{align}
    \label{theorem:1-step-greedy-for-single-step}
\end{theorem}
\jhedit{A complete proof appears in Appendix~\ref{appendix:proof-ttl}, based on geometric properties of $J_k(\delta)$ and the linear-gap assumptions.}

\begin{figure*}[!t]
    \centering
    \hfill
    \subfloat[Greedy Temporal Transfer Learning (GTTL)\label{fig:ttl-method-ttl}]{%
        \includegraphics[height=3.7cm]{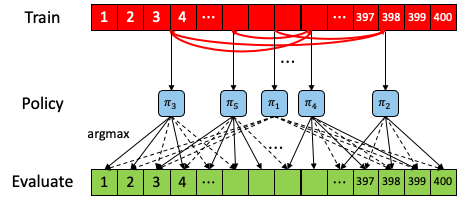}}
    \hfill
    \subfloat[Coarse-to-fine Temporal Transfer Learning (CTTL)\label{fig:ttl-method-ctl}]{%
        \includegraphics[height=3.7cm]{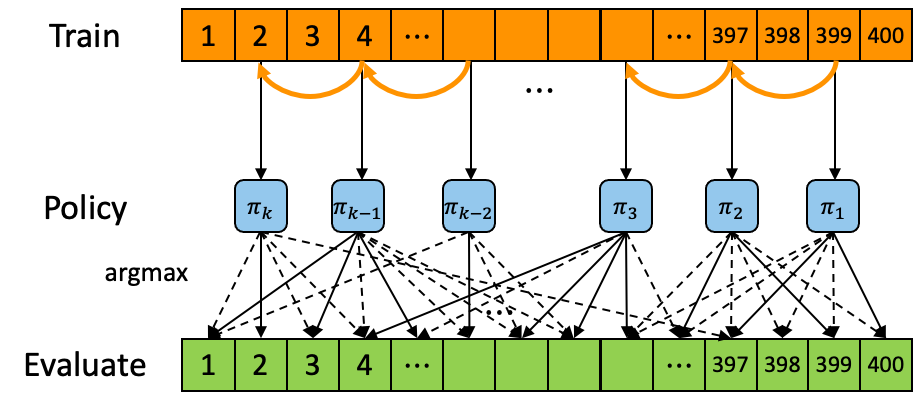}}
    \hfill
    \caption{
    \textbf{Illustrative figure of Temporal Transfer Learning (TTL) algorithms}: 
    Selecting the training task based on the TTL algorithm, evaluating each task based on the trained policies, and taking the best-performing policy for each task.
    }
    \label{fig:ttl-methods}
\end{figure*}

\jhedit{Leveraging the above assumptions and the estimated performance function, we propose the \emph{Greedy Temporal Transfer Learning (GTTL)} algorithm (Algorithm~\ref{alg:ttl-1-step}). This iterative algorithm selects the optimal hold duration to train on at each step. Initially, under Assumption~\ref{assume:same-transfer-slope}, the optimal training point is the median of the hold-duration range. For subsequent steps, the selection is guided by \Cref{theorem:1-step-greedy-for-single-step}. The transfer process continues until either the aggregate performance area is sufficiently covered or the transfer budget is exhausted. Figure~\ref{fig:ttl-methods} compares the TTL algorithms in terms of their transfer procedures.}

\begin{algorithm}[!t]
    \caption{Greedy Temporal Transfer Learning (GTTL)}
    \begin{algorithmic}[1]
    \renewcommand{\algorithmicrequire}{\textbf{Input:}}
    \renewcommand{\algorithmicensure}{\textbf{Output:}}
    \REQUIRE MDP $\jhedit{\mathcal{M}_\delta}$, Hold-duration range \jhedit{$[\delta_\text{min},\delta_\text{max}]$}, Upper bound area $A^*$, Termination criteria $\varepsilon$, Transfer budget $K$
    \ENSURE $J_k$ and $S_{k}$
    \\ \textit{Initialize} : $J_0(\delta)=0\ \forall\delta\in[\delta_\text{min},\delta_\text{max}]$, $A_k=0$, $S_k=\{\}$, $\fpi=\{\}$, $k=0$
    \WHILE {($A_k$ $\leq (1-\varepsilon)A^*$) and ($k\leq K$)}
    \STATE $\delta^{k+1} \gets \textbf{FindGreedyTransferPoint}(S_{k}, J_{k},\delta_\text{min},\delta_\text{max})$
    \STATE $S_{k+1}\gets S_{k}\cup\{\delta^{k+1}\}$
    \STATE $\pi_{k+1} \leftarrow \text{Train}(\mathcal{M}(\delta^{k+1}))$
    \STATE $\fpi \gets \fpi\cup\{\pi_{k+1}\}$
    \STATE $J_{k+1}(\delta^{k+1}) \gets J^{\pi_{k+1}}(\delta^{k+1})$
    \STATE $J_{k+1}(\delta) \gets \max(J_{k}(\delta), J_{k+1}(\delta^{k+1})-\Delta J(\delta^{k+1},\delta))$  \\
    $\quad\quad\quad\quad\quad\quad\quad\quad\quad\quad\quad\quad \forall \delta \in (\delta_\text{min},\delta_\text{max}) \backslash {\delta^{k+1}}$
    \STATE $A_{k+1} = \int_{\delta_\text{min}}^{\delta_\text{max}}{J_{k+1}(\delta)} \mathrm{d}\delta$
    \STATE $k \gets k+1$
    \ENDWHILE
    \RETURN $J_k$ and $S_{k}$
    \end{algorithmic}
    \label{alg:ttl-1-step}
\end{algorithm}

\Cref{alg:ttl-1-step} starts by initializing with the minimum and maximum hold duration, setting the performance of all hold duration to 0, initializing $J$ and $S$ to 0, and having an empty set for the policies.
It continues as long as the covered area is below a threshold or the number of source tasks is less than the budget.
For simplicity in notation, we propose substituting the whole area of $(\delta_\text{max}-\delta_\text{min})J^*$ with $A^*$.
Inside the loop, the algorithm chooses a new training task with a hold duration of $\delta^{k+1}$ and appends it to its set. It then trains a policy for this hold duration and adds it to the set of policies. 
After updating the performance with this new policy, the algorithm then calculates the area under this performance curve.
Once the loop finishes, the algorithm returns the best performance for each task ($J_k$) and a set of selected training tasks ($S_k$).
In \Cref{alg:ttl-1-step}, \Cref{alg:find-best-transfer-1-step} assists in identifying the greedy training source task, drawing insights from the shape of the estimated performance function $J$. 
This decision-making rule is grounded in \Cref{theorem:1-step-greedy-for-single-step}.

\begin{algorithm}[!t]
\caption{Find Greedy Transfer Point}
\begin{algorithmic}[1]
    \renewcommand{\algorithmicrequire}{\textbf{Function}}
    \REQUIRE FindGreedyTransferPoint($S_k, J_k,\delta_\text{min},\delta_\text{max}$)
    \STATE Cut a range of hold duration $[\delta_\text{min},\delta_\text{max}]$ into the segments split by $\delta^k \in S_k$
    \STATE \textcolor{teal}{// Choose $\delta^k$ within the segment of $[\delta_L, \delta_R]$ based on the slope of $J_k$ (\Cref{theorem:1-step-greedy-for-single-step})}
    \IF{$J_k$ is symmetric}
        \STATE $\delta^k \gets$ $\frac{\delta_L+\delta_R}{2}$
    \ELSIF{$J_k$ has positive slope}
        \STATE $\delta^k \gets$ $\frac{2\delta_L+\delta_R}{3}$
    \ELSIF{$J_k$ has negative slope}
        \STATE $\delta^k \gets$ $\frac{\delta_L+2\delta_R}{3}$
    \ENDIF
    \RETURN $\delta^k$
\end{algorithmic}
\label{alg:find-best-transfer-1-step}
\end{algorithm}

GTTL algorithm (\Cref{alg:ttl-1-step}) exhibits several noteworthy characteristics underpinning its functionality and efficiency. 
This algorithm is formulated as an anytime algorithm, meaning it can provide a valid solution even if stopped in the middle of the iterations. 
Beyond mere validity, GTTL algorithm offers performance assurances. At any given step $k$, GTTL not only provides a valid solution but also ensures a performance that is oriented towards optimization. 
For example, CTTL might struggle with finer tasks in the initial selection of the source task. 
This is because the trained policy is inherently skewed to excel in coarser tasks.
This means that while other methods like CTTL can also deliver valid results at step $k$, GTTL is specifically designed to offer a performance closer to optimal at every individual step.
This property ensures flexibility and usability under varying operational constraints, allowing continuous solution improvement with each additional source task.
This intelligent selection process ensures efficient knowledge transfer and promotes effective learning across different stages of the algorithm's execution. 

\subsection{\jhedit{Theoretical Analysis for Optimal Temporal Transfer Learning}}

\jhedit{A natural question is how effective these incremental transfer learning strategies are over multiple iterations. One optimality criterion is the best performance achieved within $K$ steps. Analogous to K-means clustering, where the number of clusters affects both granularity and computational cost, the transfer budget $K$ directly influences the quality of the solution. In practice, we are also interested in the minimum number of source tasks, denoted by $K^*(\varepsilon)$, required to achieve a performance within a suboptimality threshold $\varepsilon$. This metric quantifies the algorithm's efficiency in exploring the solution space.}

\begin{definition}[Cumulative area under the estimated performance function at each iteration]
    \jhedit{Let $A_k$ denote the cumulative area under the estimated performance function after $k$ iterations. The optimal number of steps $K^*(\varepsilon)$ is defined as the minimum $k$ such that}
    \begin{align}
        A_{K^*(\varepsilon)} &\geq (1-\varepsilon)A^*
    \end{align}
    \jhedit{where $A^*$ represents the maximum possible aggregate performance.}
    \label{def:marginal-and-cumulative-f}
\end{definition}

\jhedit{As iterations progress, the cumulative gain $A_k$ approaches $A^*$, indicating improved performance coverage with each additional source task. \Cref{def:marginal-and-cumulative-f} essentially conveys about the optimal $K$ ($K^*$) that is estimated to cover the area of $(1-\varepsilon)A^*$, having remaining area represented by the ratio of $\varepsilon$.
Hence, the definition offers insight into how many steps are required to meet the prespecified level of performance as we iterate. Although evaluating the potential coverage for each monotonic segment is critical, deriving a closed-form solution for the optimal policy is challenging due to varying segment shapes.}
\begin{figure}[!t]
    \centering
    \includegraphics[width=3in]{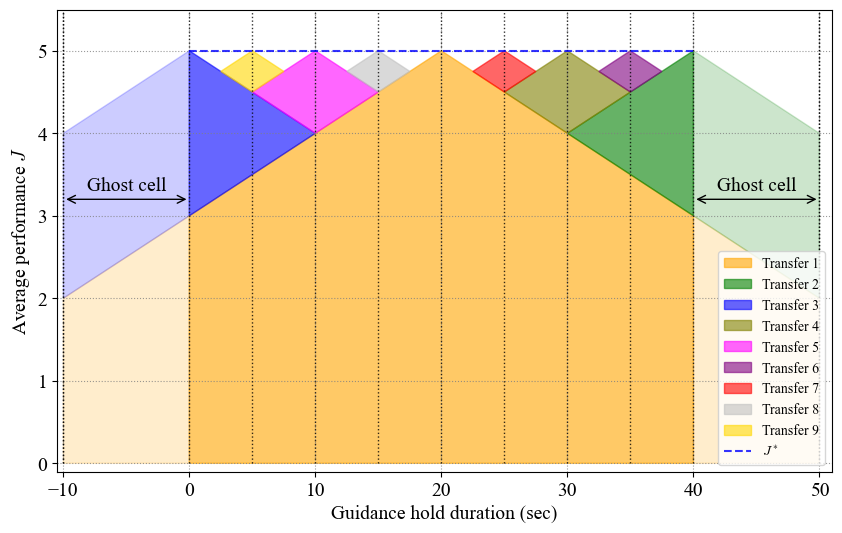}
    \caption{Illustrative figure for the \textit{lower bound of Greedy Temporal Transfer Learning (GTTL)} with the ghost cells at the end of the segments.}
    \label{fig:area-fill-lower}
\end{figure}

\jhedit{For a tractable closed-form analysis, we derive a} lower bound performance of GTTL by \jhedit{introducing} ghost cells at \jhedit{the boundaries} as depicted in \Cref{fig:area-fill-lower}.
\jhedit{In this lower-bound case, we choose $\delta_\text{max}$ and $\delta_\text{min}$ for the second and third source tasks, respectively, thereby creating a symmetric V-shaped performance profile for all sub-segments. We denote the lower bound of the cumulative area after $k$ iterations as $\Tilde{A}_{k}$.}

\begin{figure*}[!t]
    \centering
    \hfill
    \subfloat[Greedy Temporal Transfer Learning \label{fig:area-fill-f}]{
        \includegraphics[width=2.3in]{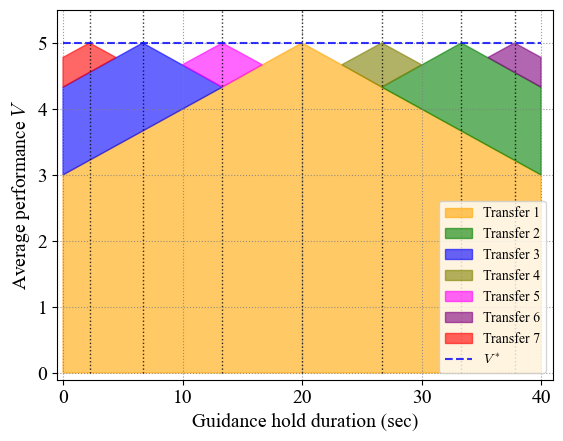}}
    \hfill
    \subfloat[Coarse-to-fine Temporal Transfer Learning \label{fig:area-fill-new-ctl}]{
        \includegraphics[width=2.3in]{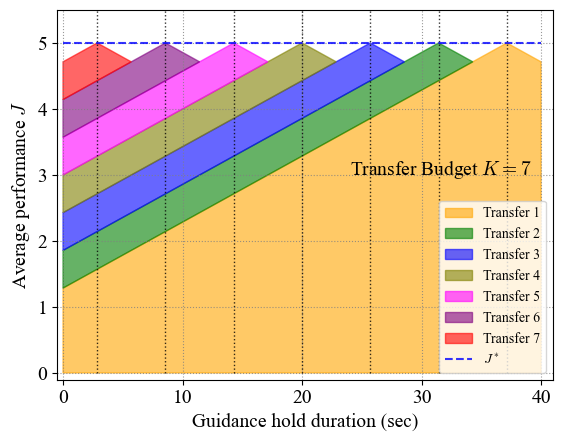}}
    \hfill
    \subfloat[Random Temporal Transfer Learning \label{fig:area-fill-new-rtl}]{
        \includegraphics[width=2.3in]{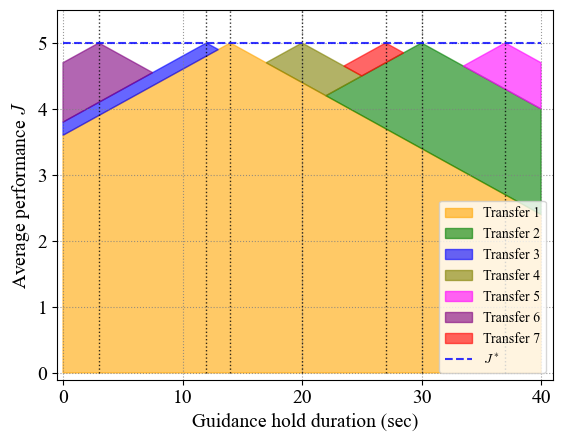}}
    \hfill
    \caption{Illustrative figures for comparing marginal area increase at each iteration by Greedy Temporal Transfer Learning (GTTL), Coarse-to-fine Temporal Transfer Learning (CTTL) for a given budget $K$, and Random Temporal Transfer Learning (RTTL).}
    \label{fig:illust-area-fill-f-g}
\end{figure*}

\begin{lemma}[Lower bound of Greedy Temporal Transfer Learning]
    \jhedit{For all iterations $k=1,\dots,K$, the cumulative area under the estimated performance function satisfies:}
    \begin{align}
        A_{k} &\geq \Tilde{A}_{k} \quad \forall k=1,...,K
    \end{align}
    \label{lemma:lower-bound-gtl}
\end{lemma}
\begin{proof}
    \jhedit{Since GTTL always selects the optimal task to maximize the area at each step, the cumulative area $A_{k}$ is always at least as large as the lower-bound area $\Tilde{A}_{k}$, which is computed using suboptimal (ghost cell) choices.}
\end{proof}

\jhedit{It follows that if there exists an integer $n$ for which $\Tilde{A}_{n} \geq (1-\varepsilon)A^*$, then $A_n \geq (1-\varepsilon)A^*$, since $A_n \geq \Tilde{A}_{n}$. In fact, at least $\frac{4\varepsilon+1}{4\varepsilon}$ steps are required to cover $(1-\varepsilon)A^*$.}

\begin{theorem}[The number of source tasks required to cover the area]
    \jhedit{If there exists an integer $n$ such that $\Tilde{A}_{n}\geq (1-\varepsilon)A^*$, then $A_n\geq (1-\varepsilon)A^*$. Moreover, at least $\frac{4\varepsilon+1}{4\varepsilon}$ iterations are necessary to achieve this performance threshold.}
    \label{theorem:area-fill-ghost}
\end{theorem}

We prove \Cref{theorem:area-fill-ghost} by leveraging the lower bound cumulative area of GTTL as outlined in \Cref{lemma:lower-bound-gtl}, which has a streamlined expression of $A_n$.
The comprehensive proof is provided in Appendix~\ref{appendix:proof-area-fill-ghost}.

\subsection{Bounded Suboptimality}
\jhedit{If the transfer budget is known in advance, a more structured algorithm than GTTL can be designed. This motivates our introduction of \emph{Coarse-to-fine Temporal Transfer Learning (CTTL)}, which selects source tasks uniformly across the hold-duration range. CTTL begins with coarser tasks and progressively transitions to finer ones. For example, with a budget of $7$ source tasks and a hold-duration range from $1$ to $40$, training might begin with a hold duration of approximately $37.14$, then proceed to $31.43$, $25.71$, $20$, $14.29$, $8.57$, and finally $2.86$, reflecting a diminishing granularity (see Figure~\ref{fig:area-fill-new-ctl}).}

\jhedit{The advantage of starting with coarser tasks lies in their limited effective horizon, making them easier to solve. Prior work suggests that a coarse-to-fine transfer learning approach is beneficial when adapting a pre-trained policy from a coarser to a finer task \cite{wei_coarse--fine_2018, wang_coarse--fine_2023}.}

\begin{algorithm}[!ht]
    \caption{Coarse-to-fine Temporal Transfer Learning (CTTL)}
    \begin{algorithmic}[1]
    \renewcommand{\algorithmicrequire}{\textbf{Input:}}
    \renewcommand{\algorithmicensure}{\textbf{Output:}}
    \REQUIRE MDP $\jhedit{\mathcal{M}_\delta}$, Range of hold duration $[\delta_\text{min},\delta_\text{max}]$, Transfer budget $K$
    \ENSURE $J_k$ and $S_{k}$
    \\ \textit{Initialize} : $J_0(\delta)=0\ \forall\delta\in[\delta_\text{min},\delta_\text{max}]$, $S_k=\{\}$, $\fpi=\{\}$, $k=0$
    \WHILE {$k\leq K$}
        \STATE $\delta^{k+1}=\delta_\text{max}-\frac{2k+1}{2K}(\delta_\text{max}-\delta_\text{min})$
        \STATE $S_{k+1}\gets S_k\cup\{\delta^{k+1}\}$
        \STATE $\pi_{k+1} \leftarrow \text{Train}(\mathcal{M}(\delta^{k+1}))$
        \STATE $\fpi\gets\fpi\cup\{\pi_{k+1}\}$
    \STATE $J_{k+1}(\delta^{k+1}) \gets J^{\pi_{k+1}}(\delta^{k+1})$
    \STATE $J_{k+1}(\delta) \gets \max(J_{k}(\delta), J_{k+1}(\delta^{k+1})-J_{\delta^{k+1} \to \delta})$  \\
    $\quad\quad\quad\quad\quad\quad\quad\quad\quad\quad\quad\quad \forall \delta \in (\delta_\text{min},\delta_\text{max}) \backslash {\delta^{k+1}}$
    \STATE $k \gets k+1$
    \ENDWHILE
    \RETURN $J_k$ and $S_{k}$
    \end{algorithmic}
    \label{alg:cttl}
\end{algorithm}

\Cref{lemma:optimality-ctl} states the optimality of the CTTL algorithm, which selects its subsequent transfer task contingent on the allocated transfer budget $K$.

\begin{lemma}[Optimality of Coarse-to-fine Temporal Transfer Learning]
    \jhedit{Coarse-to-fine Temporal Transfer Learning (CTTL) algorithm establishes the optimality under Assumptions~\ref{assume:constant-upperbound-j}--\ref{assume:upperbound-J}. Starting from the coarsest task with hold duration $(\delta_\text{max}-\frac{\delta_\text{max}-\delta_\text{min}}{2K})$, CTTL transitions to finer tasks with uniform spacing of $\frac{\delta_\text{max}-\delta_\text{min}}{K}$. The optimal estimated performance after $K$ iterations is given by:}
    \begin{equation}
        A^\text{CTTL}_K=(1-\frac{1}{4K})\theta(\delta_\text{max}-\delta_\text{min})^2
    \end{equation}
    \label{lemma:optimality-ctl}
\end{lemma}
\jhedit{The optimality of CTTL can be established via the equality condition of the Cauchy--Schwarz inequality; see Appendix~\ref{appendix:proof-ctl-optimality} for details.}

\jhedit{In practice, when the optimal number of source tasks $K^*(\varepsilon)$ is known, CTTL tends to outperform GTTL. However, precise knowledge of the transfer budget is often unavailable, making it important to quantify the suboptimality gap between GTTL and the oracle-like CTTL. \Cref{theorem:suboptimality-ttl} provides bounds on the suboptimality of GTTL relative to CTTL for a given transfer budget $K$.}
\begin{theorem}[Suboptimality of Greedy Temporal Transfer Learning]
    \jhedit{The suboptimality gap of GTTL relative to CTTL is bounded by}
    \begin{align}
        \begin{cases}
            \frac{1}{4K(K-1)}\theta(\delta_\text{max}-\delta_\text{min})^2 & \text{ for } K=2^i+1\\
            \frac{1}{2(K-1)^2}\theta(\delta_\text{max}-\delta_\text{min})^2 & \text { otherwise}
        \end{cases}
    \end{align}
    where $i\in \bN_0$.
    \label{theorem:suboptimality-ttl}
\end{theorem}

\jhedit{
The detailed proof is provided in Appendix~\ref{appendix:proof-ttl-suboptimality}, where the suboptimality bounds of GTTL relative to CTTL are derived for the two cases specified above.
}

\section{Simulation Experiments}
This section elucidates the simulation experiments conducted to address our primary research questions. The main purpose of our investigation is to explore the potential of human-compatible control serving as an immediate surrogate for AVs and to verify the degree to which such control can optimize traffic performance at a system level. We conducted many experimental trials in various environments to obtain valuable answers to these essential questions.

\begin{figure}[!t]
    \centering
    \includegraphics[width=3.5in]{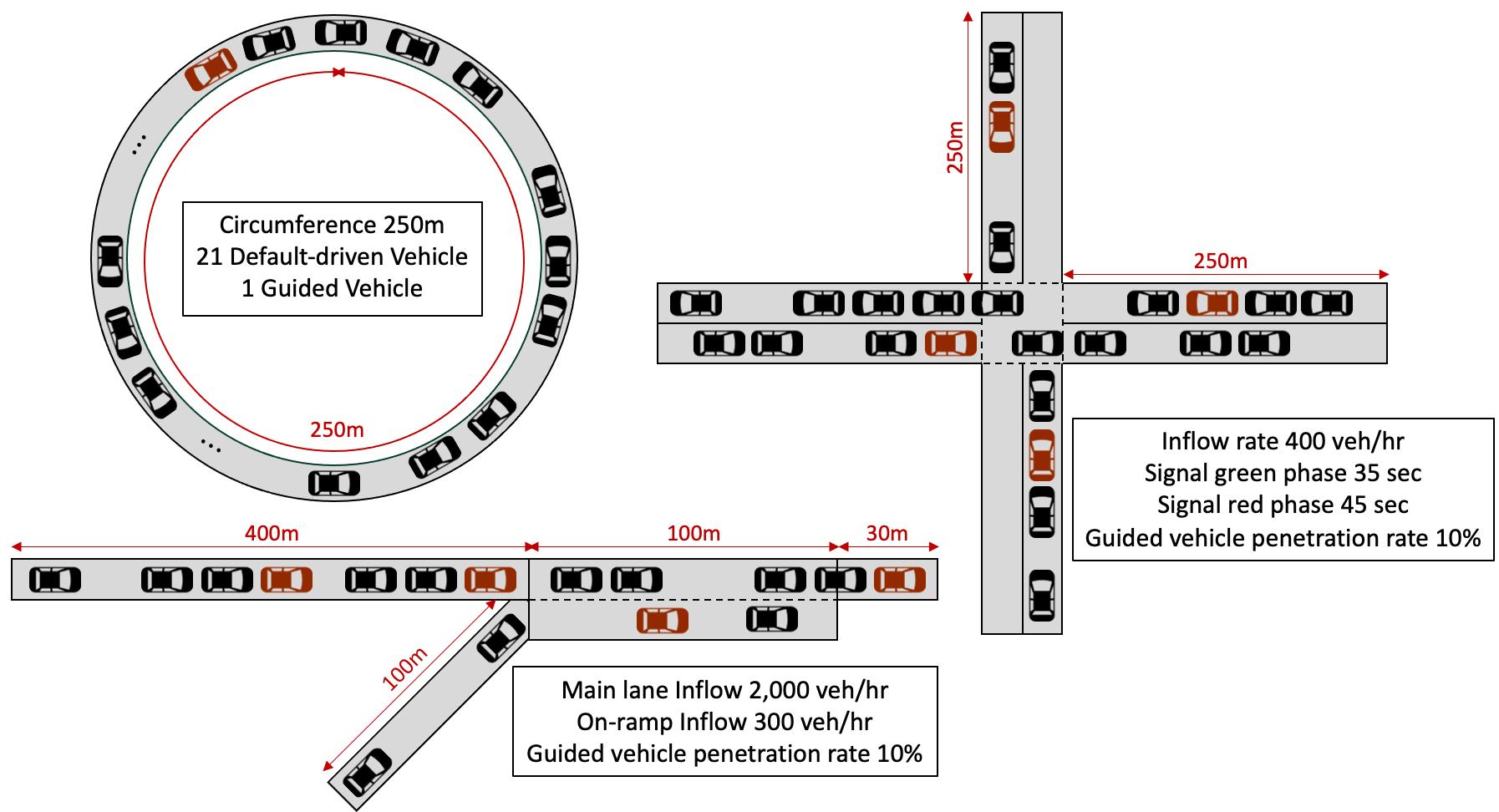}
    \caption{\textbf{Modular road networks.} Three traffic scenarios for mixed autonomy roadway settings: single-lane ring (top left), highway ramp (bottom), and signalized intersection (top right).}
    \label{fig:three-scenarios}
\end{figure}
\subsection{Modular Road Networks}

\noindent In mixed-autonomy roadway settings, we delve into various traffic scenarios as explored in prior works \cite{yan_reinforcement_2021, wu_flow_2022, yan_unified_2022}, including single-lane ring, highway ramp, and signalized intersection networks, as depicted in Figure~\ref{fig:three-scenarios}. 
Each scenario has distinct objectives; for instance, the single-lane ring and intersection aim to elevate all vehicles' average velocity, while the highway ramp scenario focuses on increasing the outflow given a constant inflow. The signalized intersection scenario employs a multitask RL strategy, simulating varied penetration rates to accommodate different levels of human-guided vehicle presence, \jhedit{with an evaluation of a penetration rate of 0.1 to assess the RL policy's performance.}
\jhedit{Default-driven vehicles, which are not equipped with the guidance system, adhere to the Intelligent Driver Model (IDM) for car-following \cite{treiber_congested_2000}.}

\jhedit{In the single-lane scenario, the state space is defined by the ego vehicle's speed, the leading vehicle's speed, and the headway. For the highway ramp scenario, the state includes the ego vehicle's speed, relative positions and speeds of the leading and following vehicles in the same lane, and those of the following vehicle on the ramp. In the intersection scenario, the state comprises the ego vehicle's speed, the distance remaining to the intersection, the traffic signal phase, as well as the relative positions and speeds of the leading, following, and adjacent vehicles, the current speed limit, and lane and road identifiers.

The action space varies with the type of guidance employed. For acceleration guidance, a continuous action space ranging from $-1$ to $1$ is utilized, directly influencing the vehicle's acceleration up to a \jhedit{maximum} of 2.5$\text{m/s}^2$. For speed guidance, a discrete action space is defined, with ten actions ranging from $0$ to $1$, where the chosen action is scaled by the speed limit to determine the vehicle's target speed.}
\jhedit{The respective reward functions are tailored to the objectives of each scenario. The reward functions for the single-lane ring and highway ramp are the average speed and throughput of the system, respectively. In the signalized intersection scenario, the reward function is defined as the average speed of all vehicles alongside other factors such as stopping time, abrupt acceleration, and fuel consumption. A thorough examination of these scenarios and reward formulations is provided in Appendix~\ref{appendix:exp-modular-road-network}.}

\subsection{Experimental Setup}
\noindent We utilize the microscopic traffic simulation called Simulation of Urban MObility (SUMO) \cite{SUMO2018} v.1.16.0 and \jhedit{its accompanying Python API, TRACI, to establish a dynamic link between our algorithmic framework and the SUMO environment. This integration is critical for implementing and testing our traffic management strategies in a controlled, simulated setting}.
The experiments used the MIT Supercloud with 48 CPU cores \cite{reuther_interactive_2018}.
\jhedit{For our numerical experiments, we employed the Trust Region Policy Optimization (TRPO) algorithm \cite{schulman_trust_2017}, coupled with a Multilayer Perceptron (MLP) neural network architecture featuring two hidden layers, each with 64 units, and employing the $\tanh$ activation function.}
\jhedit{We trained and tested two different types of guidance: continuous action space for acceleration guidance and discretized action space for speed guidance.}
We evaluated the system's performance over a range of guidance hold duration $\delta \in [0.1, 40]$. 
In our simulation experiments, we simplify the analysis by setting $\delta_\text{min}=0$, and we round the calculated $\delta$ to the nearest multiple of 10 \jhedit{when selecting the source training task}.
The detailed experimental setup is explained in \Cref{fig:three-scenarios} and Appendix~\ref{appendix:exp-modular-road-network}.

\begin{figure*}[!t]
    \centering
    \subfloat[Single-lane ring\label{fig:ring-scratch}]{%
        \includegraphics[height=4.5 cm]{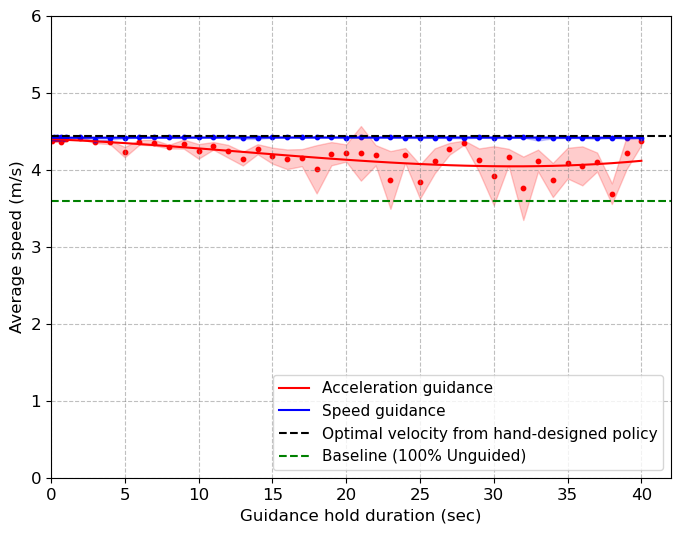}}
    \hfill
    \subfloat[Highway ramp\label{fig:ramp-scratch}]{%
        \includegraphics[height=4.5cm]{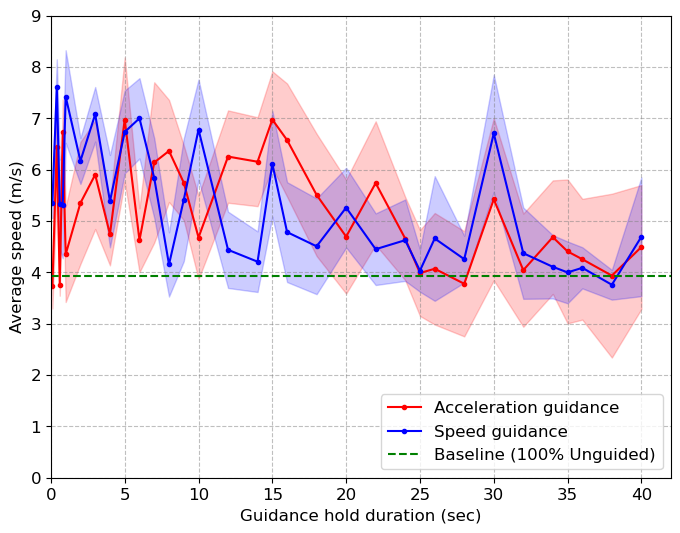}}
    \hfill
    \subfloat[Signalized intersection\label{fig:intersection-scratch}]{%
        \includegraphics[height=4.5cm]{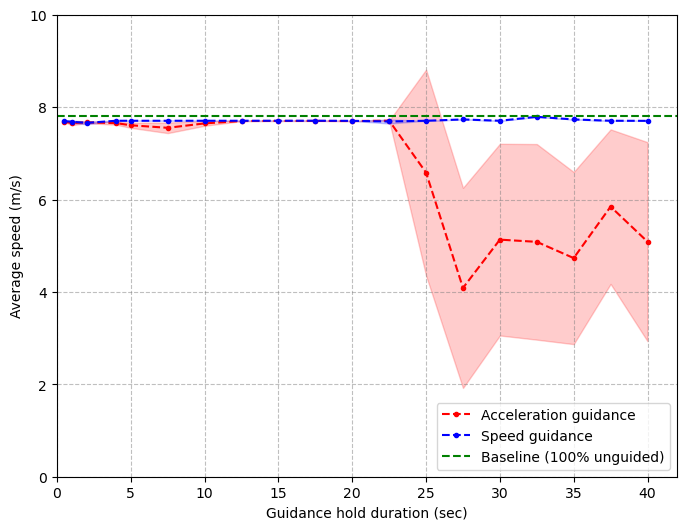}}
    \caption{System performance (average speed of all vehicles) for three traffic scenarios in mixed autonomy roadway settings. \jhedit{Each task with different hold durations is trained exhaustively.}}
    \label{fig:three-scratches}
\end{figure*}

\subsection{Baselines}
We compare our TTL approaches with several baselines. 
These baselines represent various strategies for learning and transfer in the context of coarse-grained advisory autonomy tasks.

\subsubsection{100\% Unguided}
In this baseline, all vehicles are unguided, following the Intelligent Driver Model (IDM) car-following models. This scenario represents a completely decentralized system without any reinforcement learning.
\subsubsection{Oracle Transfer}
\jhedit{The Oracle Transfer benchmark is an idealized scenario where we train separate models for each possible source task. Once trained, we execute a zero-shot transfer by applying each model to every target task, selecting the most successful model for each.}

\subsubsection{Exhaustive RL}
This strategy represents the classic approach to machine learning, where \jhedit{each task is trained individually and evaluated with its corresponding in-distribution trained model.} The performance is evaluated by calculating the average performance across all tasks. \jhedit{\Cref{fig:three-scratches} illustrates the average speed of all vehicles for three traffic networks trained exhaustively.}

\subsubsection{Multitask RL}
\jhedit{The multitask reinforcement learning framework is designed to simultaneously train a single policy across various tasks \cite{belletti_expert_2018, yan_unified_2022}, here specifically by varying the guidance hold duration between 1 to 40 seconds. Each worker in the training process is assigned a specific task and contributes to a shared experience buffer. Upon completion of rollouts by all workers, the policy is updated based on the collective data in the buffer. This approach aims to explore the potential synergies and trade-offs that arise when a policy is exposed to multiple tasks during the learning process, potentially leading to more robust and generalizable policies.}
\subsubsection{Random Temporal Transfer Learning (RTTL)}
Random Temporal Transfer Learning (RTTL) chooses tasks for transfer from a pool of temporal tasks at random. This scenario represents a non-deterministic transfer learning strategy and serves as a stochastic comparison point for our deterministic GTTL approach. 
From all the policies from the source task at each iteration, we select the top-performing one for tasks with varying hold duration. (\Cref{fig:area-fill-new-rtl})

\subsection{\jhedit{Temporal Transfer Learning (TTL) results}}
\begin{figure*}[!t]
    \centering
    \subfloat[Single lane ring with acceleration guidance\label{fig:ring-acc-transfered}]{%
        \includegraphics[height=4.7cm]{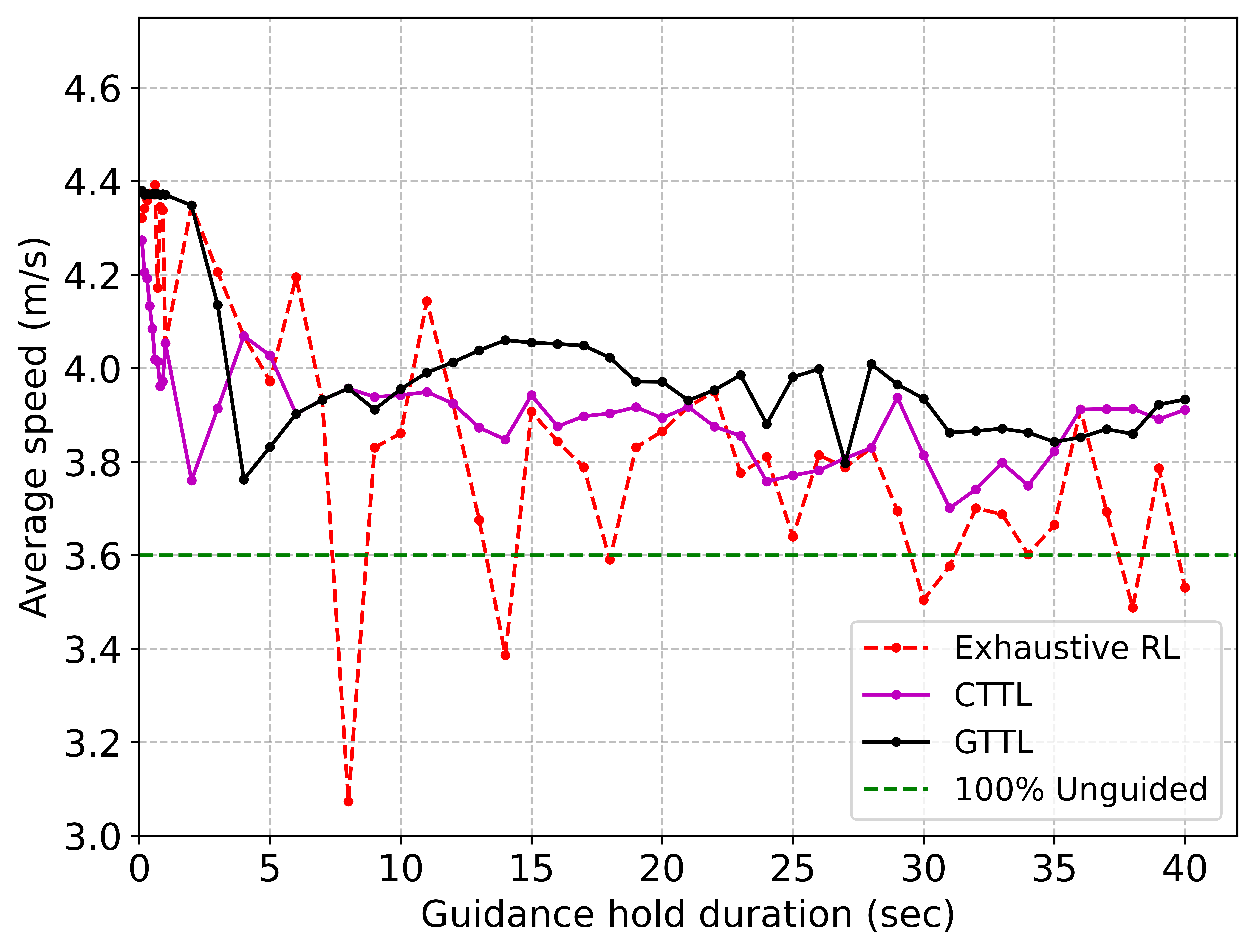}}
    \subfloat[Highway ramp with acceleration guidance\label{fig:ramp-acc-transfered}]{%
        \includegraphics[height=4.7cm]{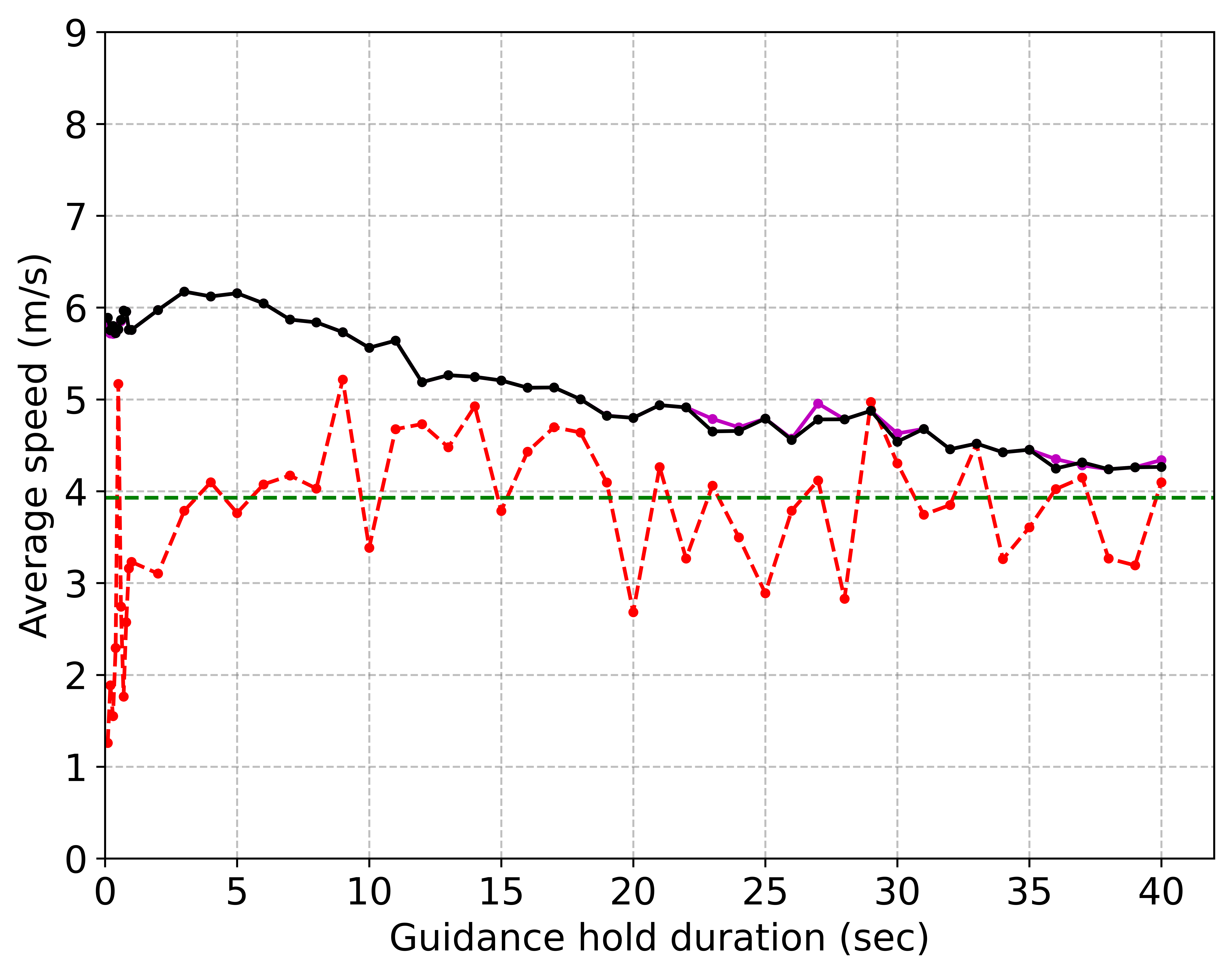}}
    \subfloat[Signalized intersection with acceleration guidance\label{fig:intersection-acc-transfered}]{%
        \includegraphics[height=4.7cm]{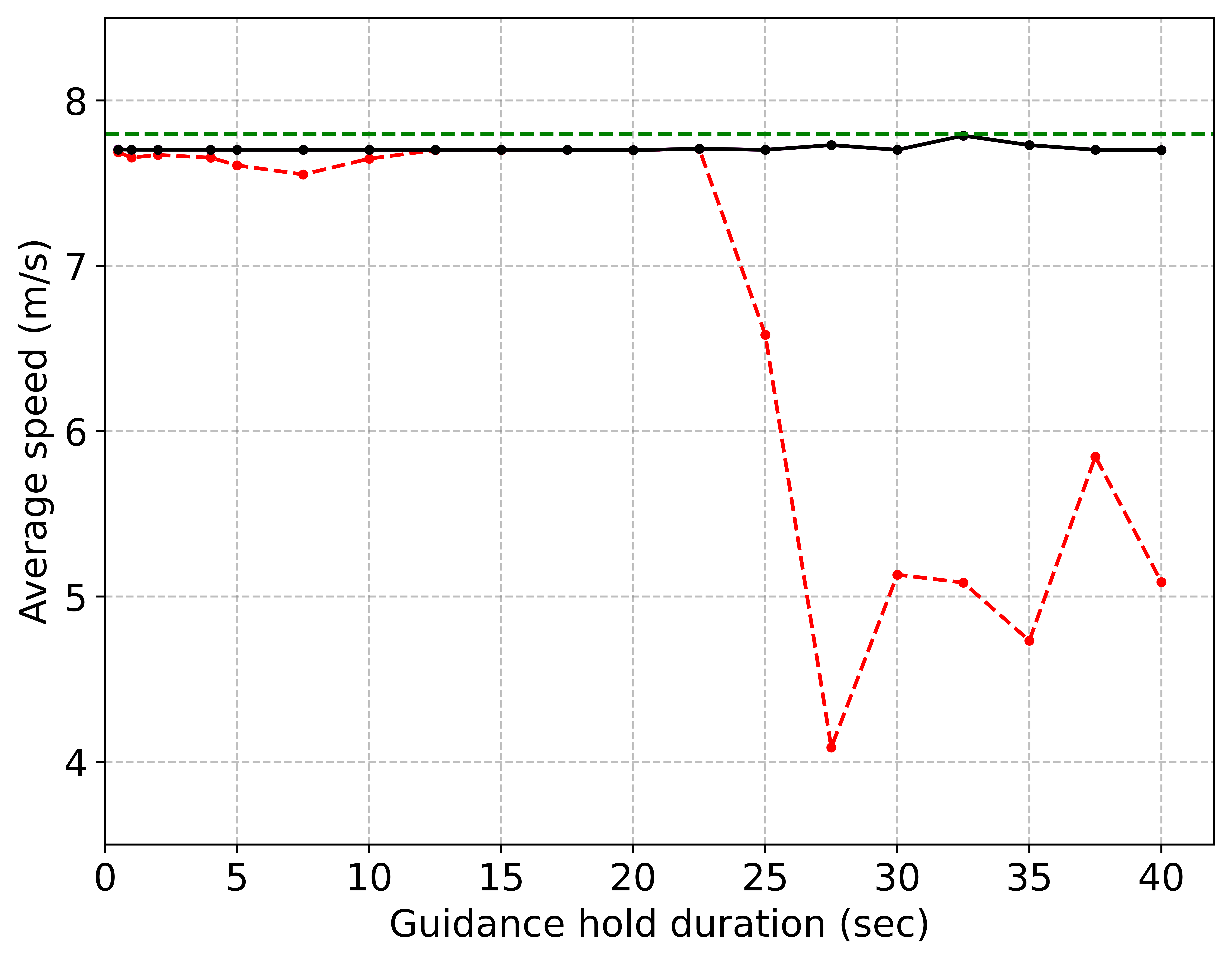}}
    \hfill
    \subfloat[Single lane ring with speed guidance\label{fig:ring-vel-transfered}]{%
        \includegraphics[height=4.7cm]{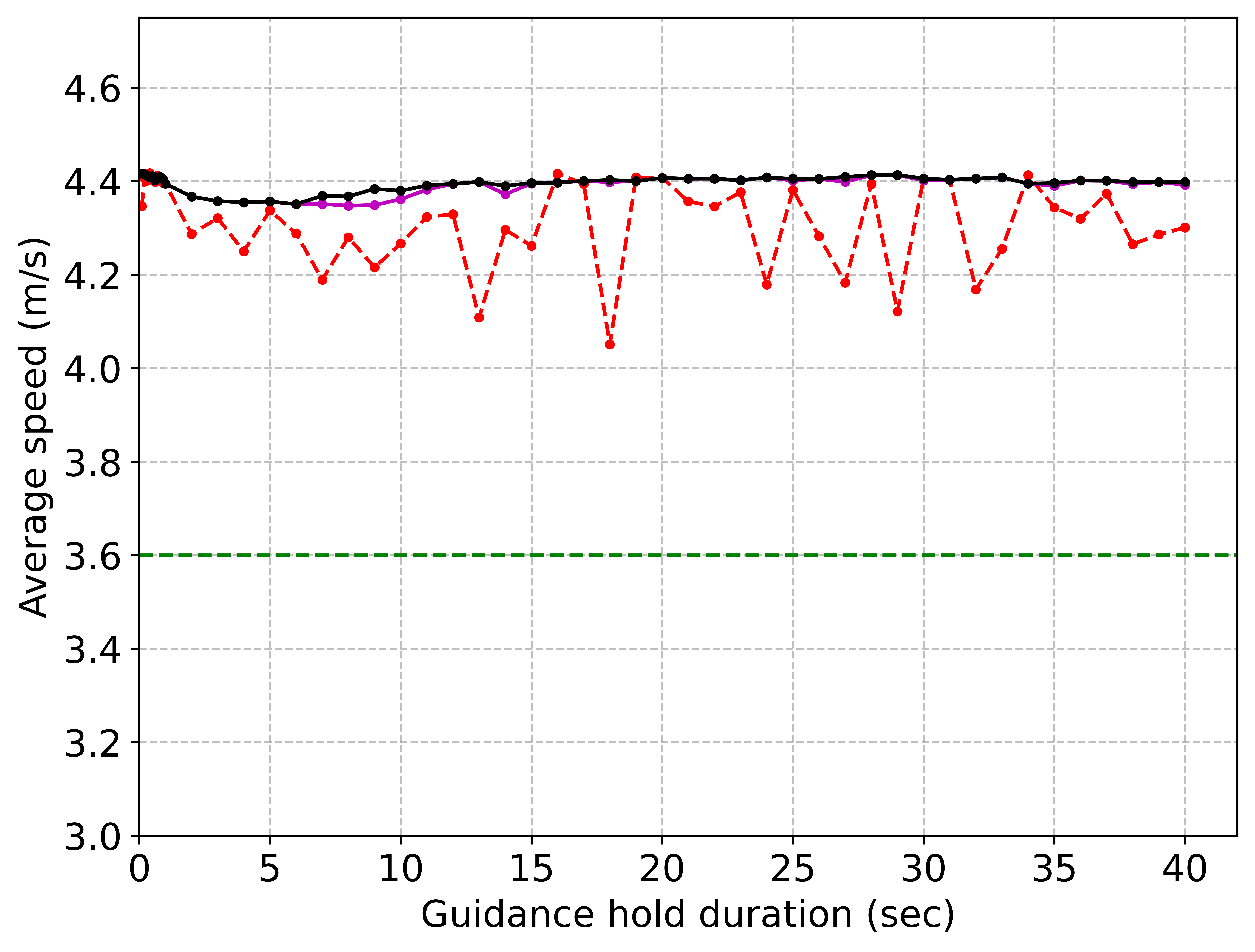}}
    \subfloat[Highway ramp with speed guidance\label{fig:ramp-vel-transfered}]{%
        \includegraphics[height=4.7cm]{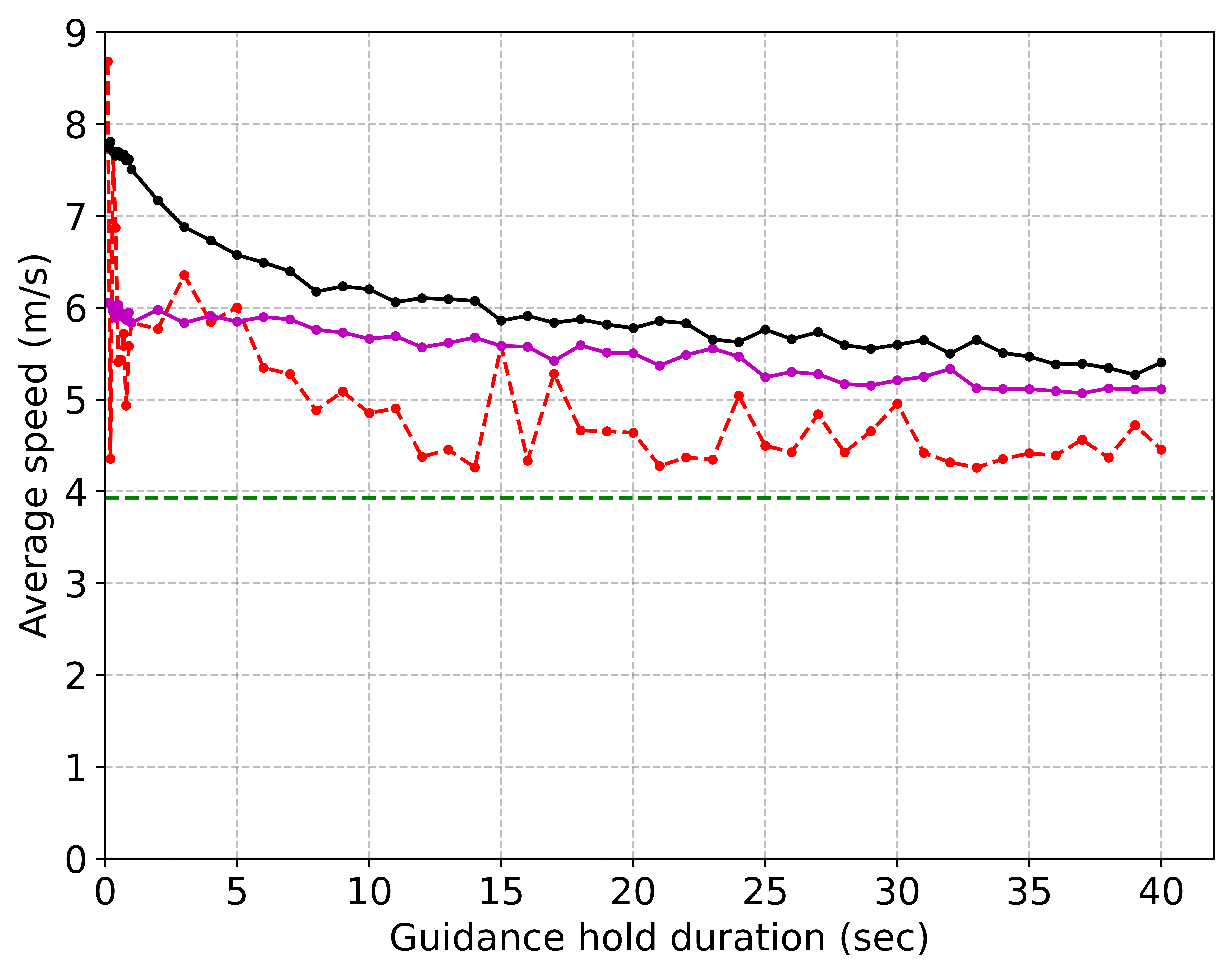}}
    \subfloat[Signalized intersection with speed guidance\label{fig:intersection-vel-transfered}]{%
        \includegraphics[height=4.7cm]{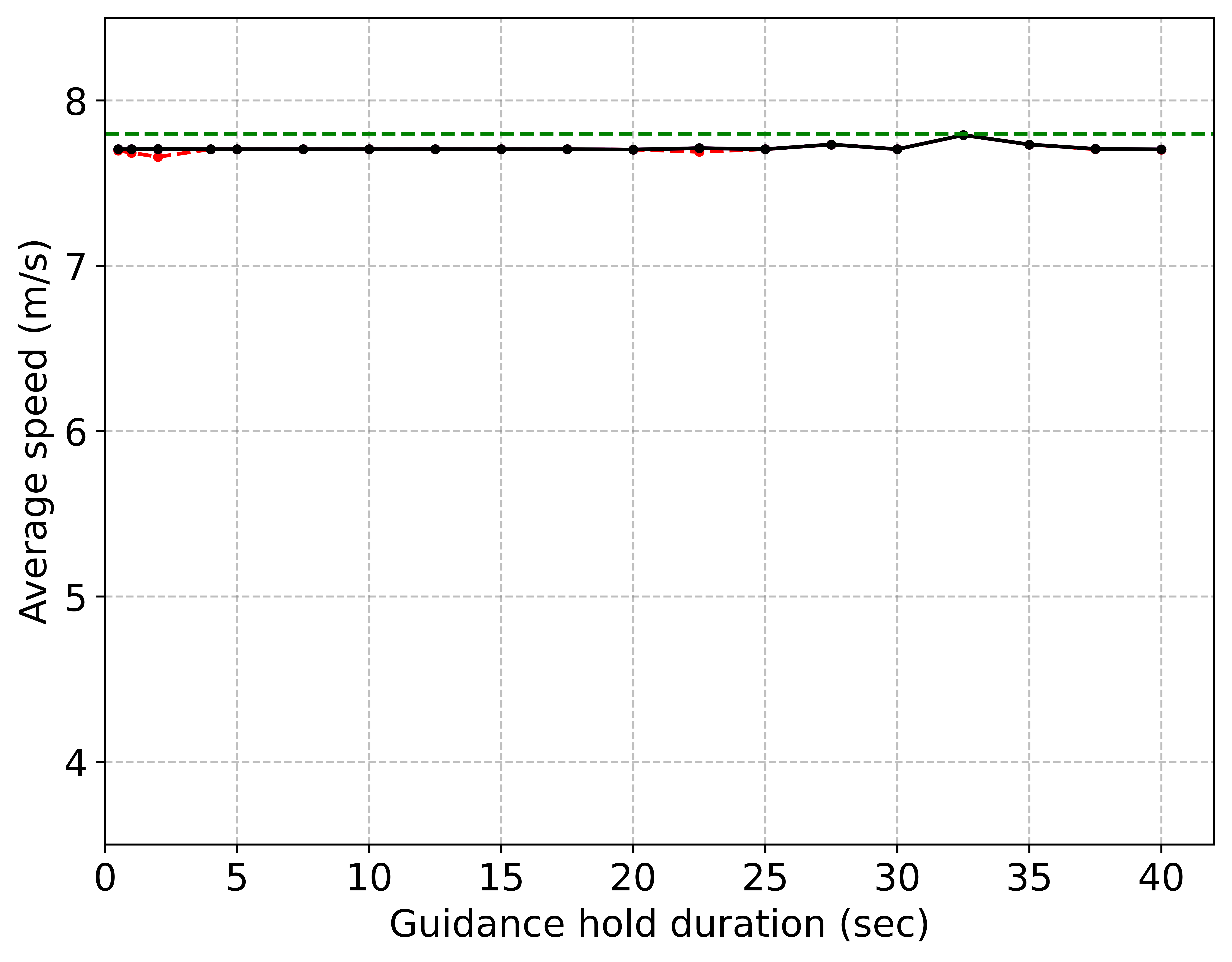}}
    \caption{System performance of Temporal Transfer Learning algorithms (GTTL and CTTL) compared to the exhaustive RL. \jhedit{The advisory system used either acceleration or speed guidance.}}
    \label{fig:three-transfered}
\end{figure*}

\begin{figure*}[!t]
    \centering
    \subfloat[Single lane ring with acceleration guidance\label{fig:ttl-ring-acc}]{
        \includegraphics[height=4.7cm]{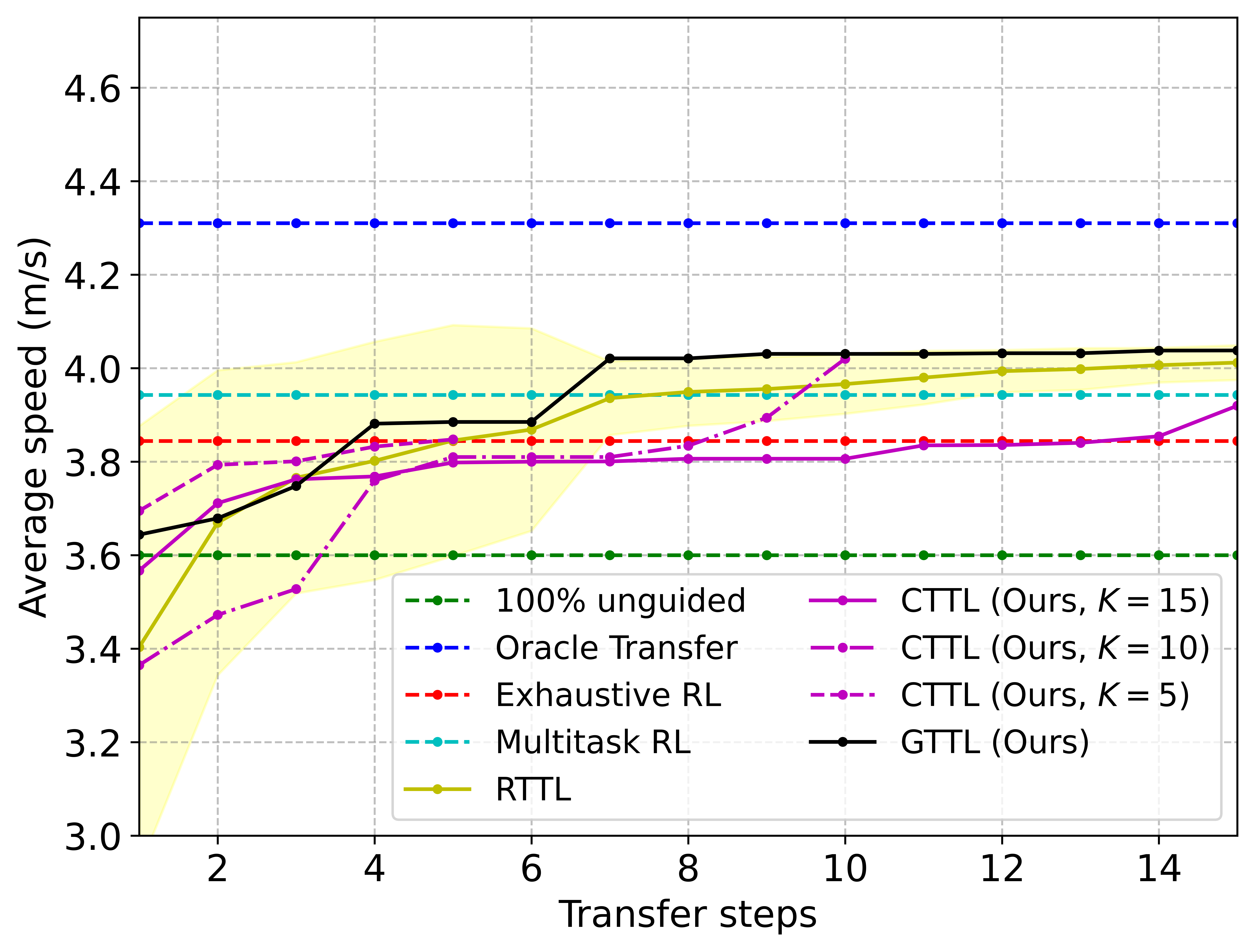}}
    \subfloat[Highway ramp with acceleration guidance\label{fig:ttl-ramp-acc}]{%
        \includegraphics[height=4.7cm]{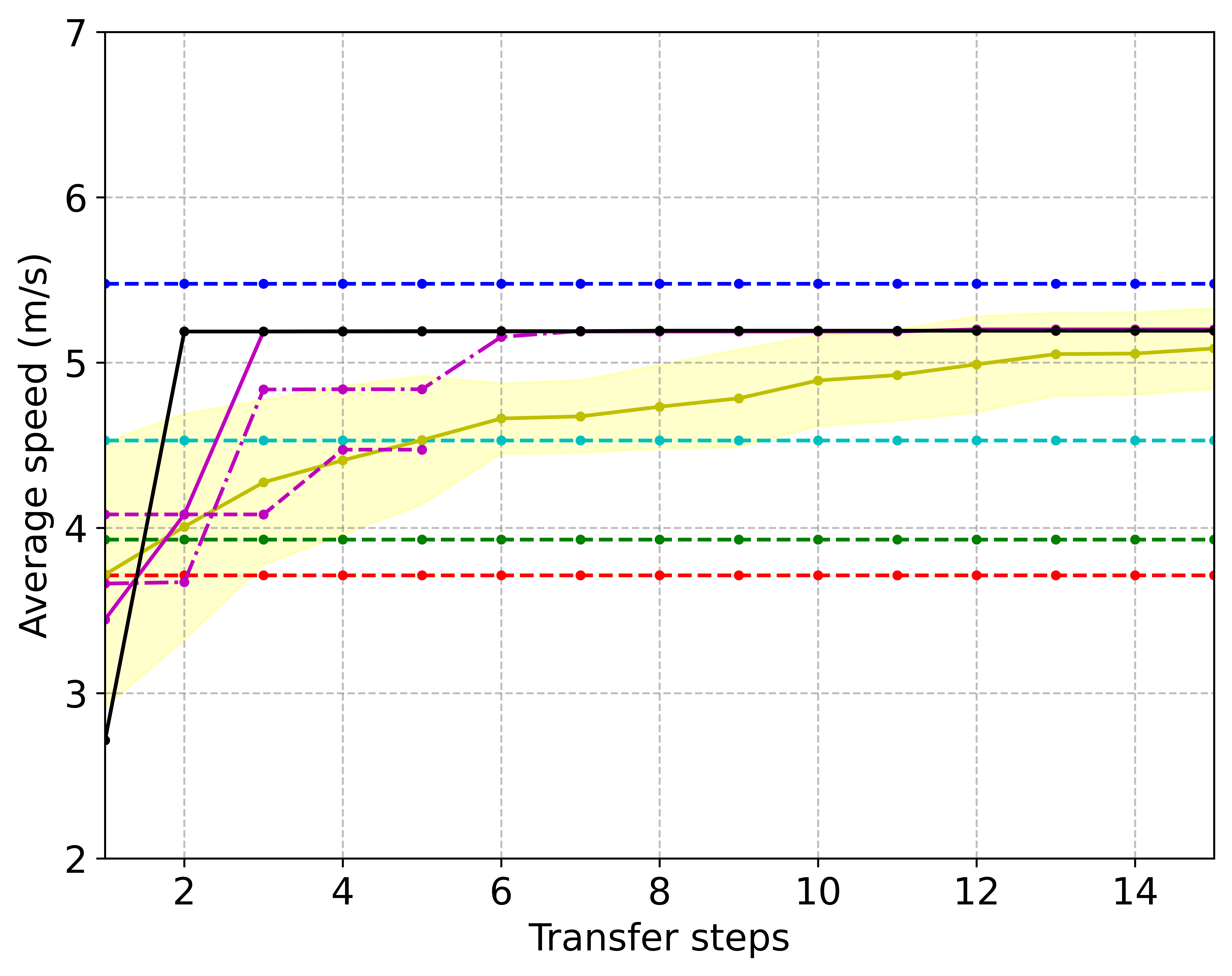}}
    \subfloat[Signalized intersection with acceleration guidance\label{fig:ttl-inter-acc}]{%
        \includegraphics[height=4.7cm]{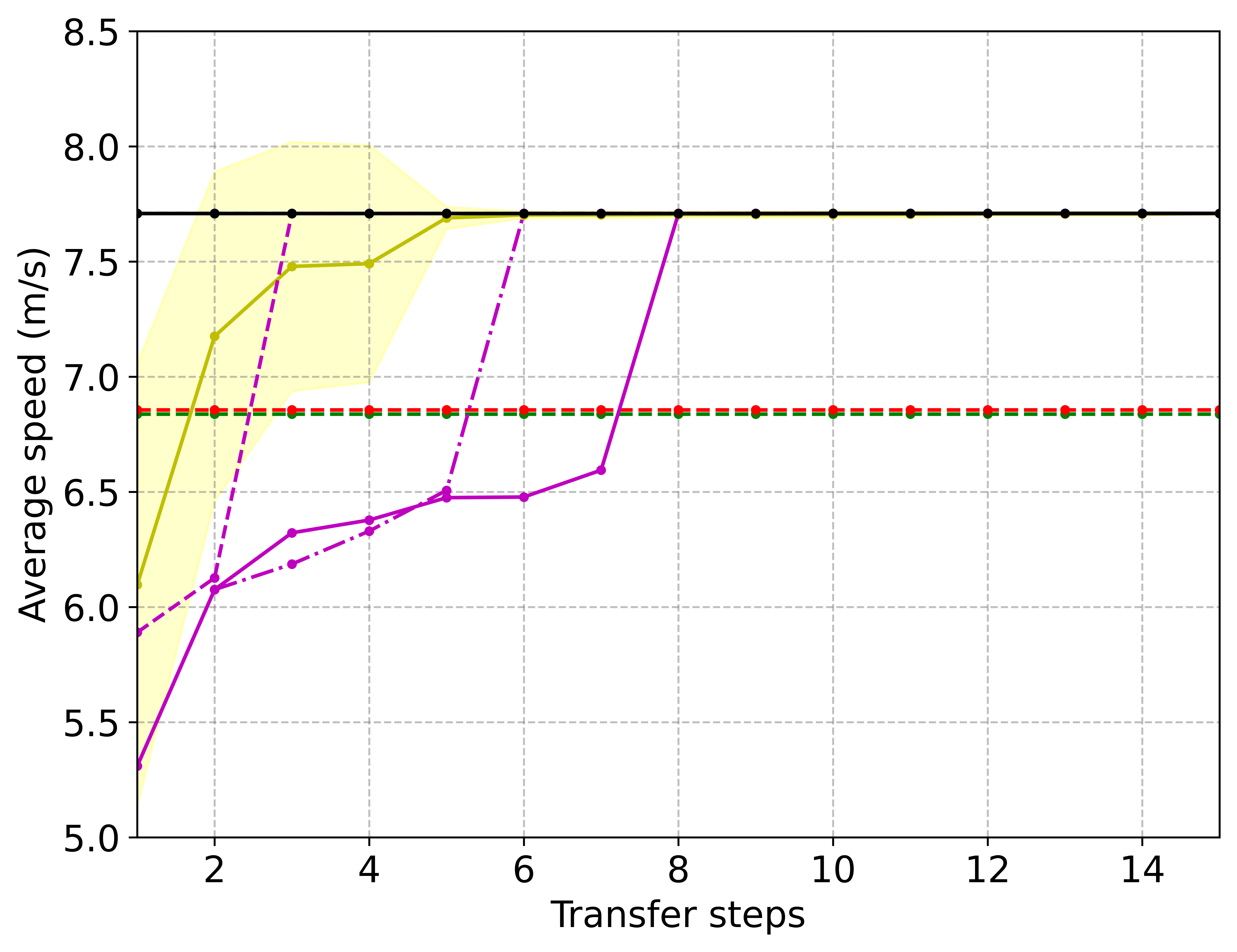}}
    \hfill
    \subfloat[Single lane ring with speed guidance\label{fig:ttl-ring-vel}]{
        \includegraphics[height=4.7cm]{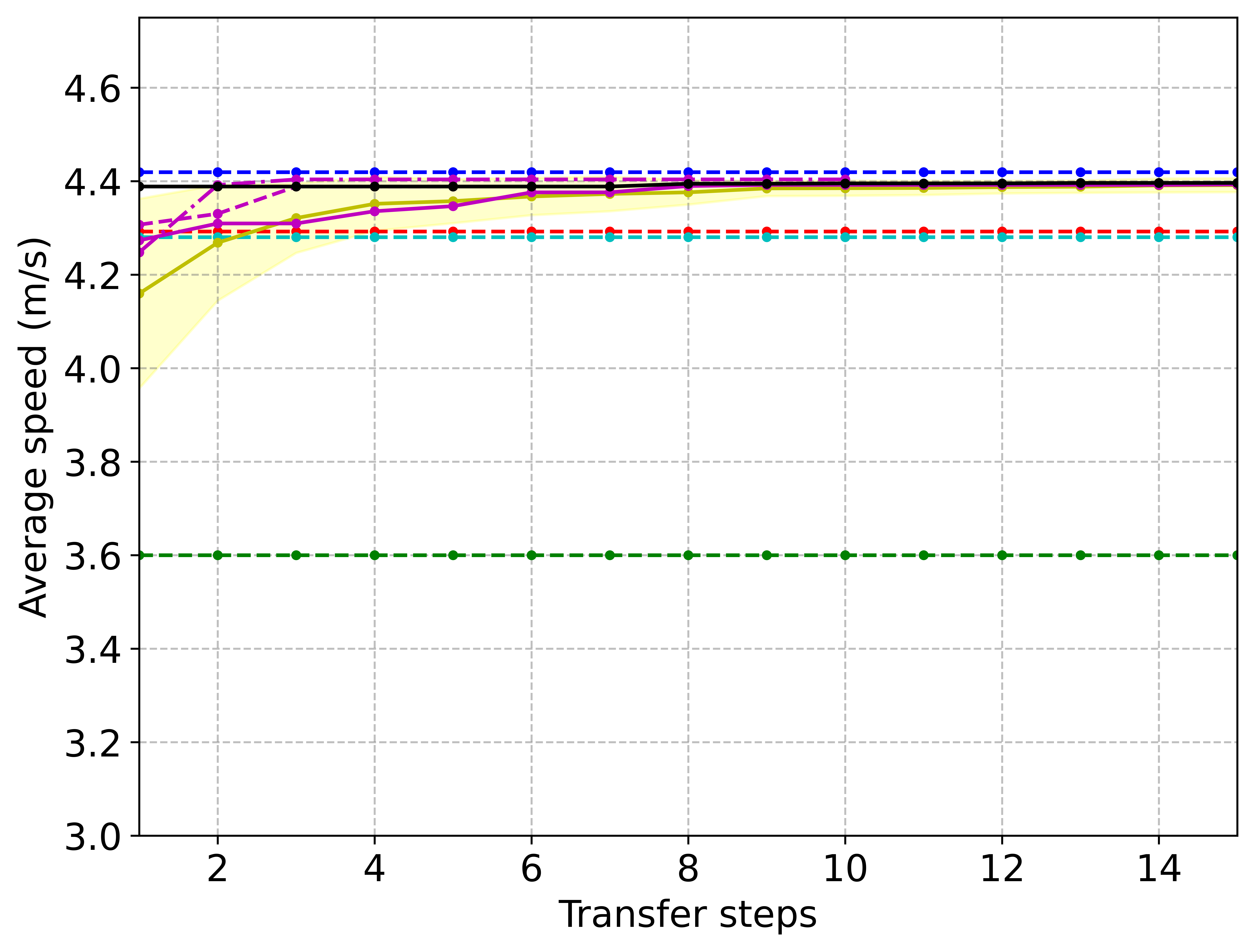}}
    \subfloat[Highway ramp with speed guidance\label{fig:ttl-ramp-vel}]{%
        \includegraphics[height=4.7cm]{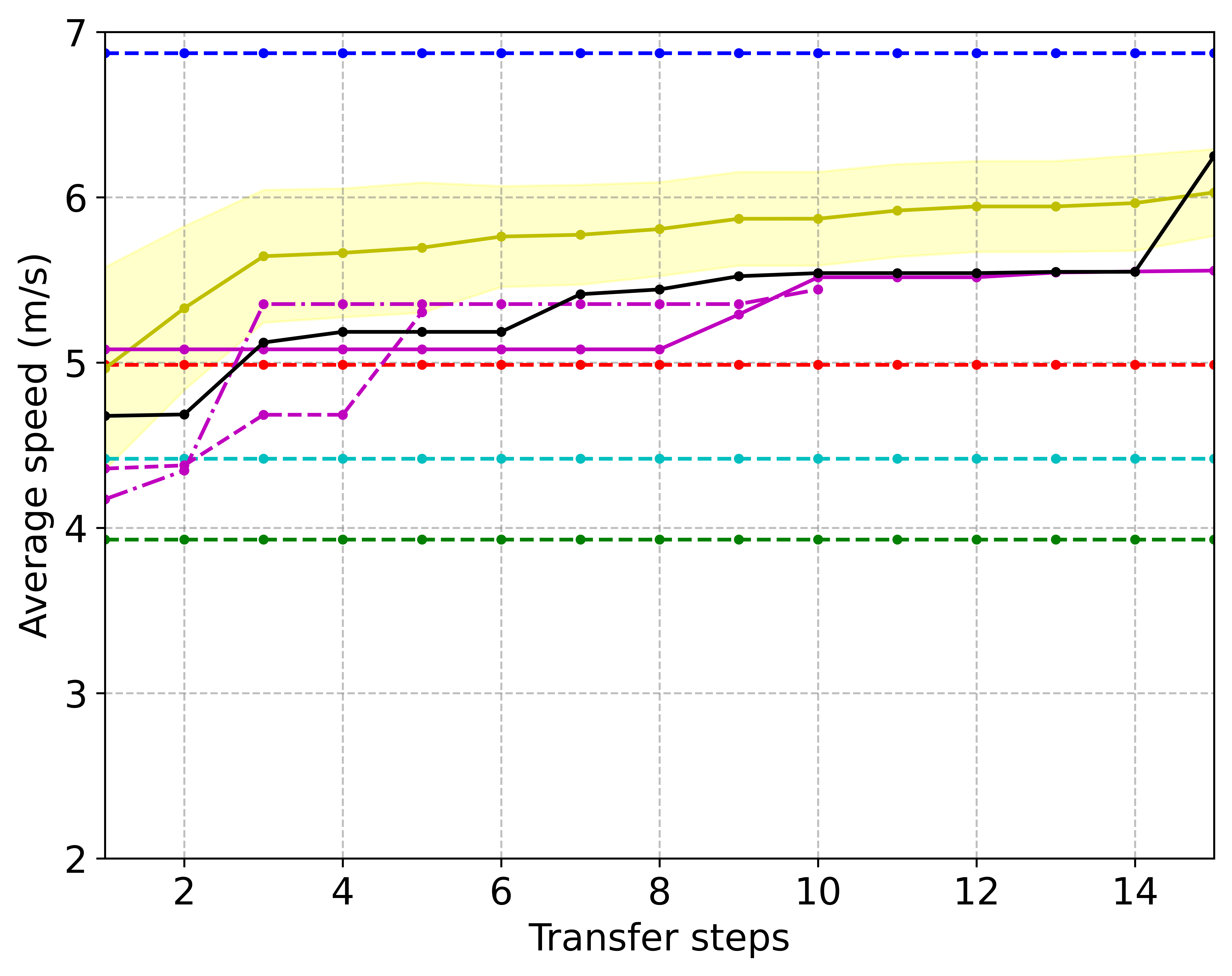}}
    \subfloat[Signalized intersection with speed guidance\label{fig:ttl-inter-vel}]{%
        \includegraphics[height=4.7cm]{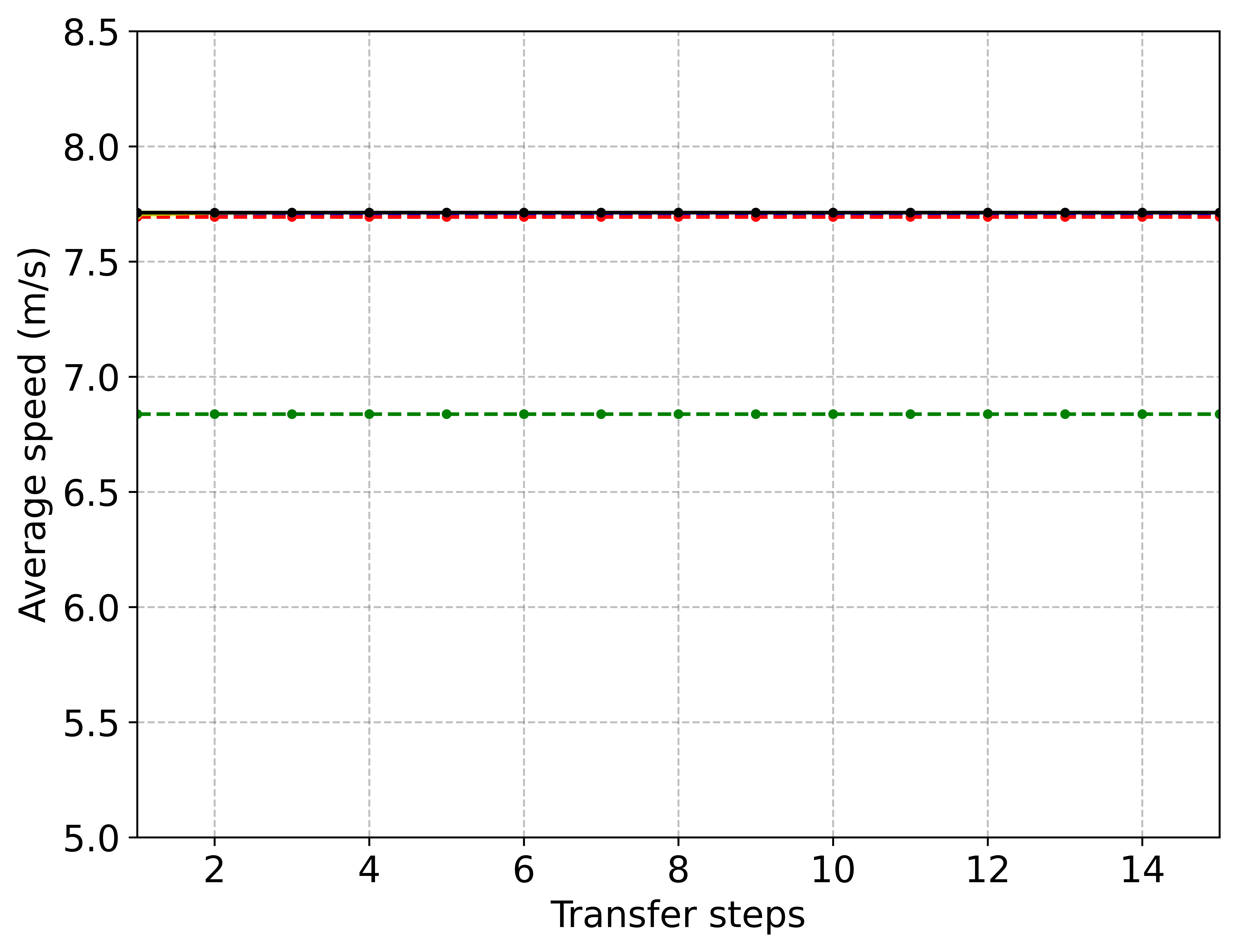}}
    \caption{
    System performance comparison of Temporal Transfer Learning (TTL) with various baselines in three different traffic scenarios and two different guidance types: unguided vehicles (representing a completely decentralized system), oracle transfer (showcasing zero-shot transfer capabilities), exhaustive RL (traditional separate model per task approach), multitask reinforcement learning (incorporating multiple tasks into the learning process), and transfer learning strategies such as greedy (choosing the 1-step greedy source training task) and coarse-to-fine (transferring from coarser to finer tasks progressively).
    }
    \label{fig:three-ttl}
\end{figure*}

\begin{table*}[!t]
\caption{Comparison of Average Speeds Achieved by Different Training Methods Across Various Traffic Scenarios and Guidance Types}
\label{table:simulation-results}
\centering
\begin{tabular}{llccrrrrrr}
\cline{5-10}
 &  & \multicolumn{1}{l}{}  & \multicolumn{1}{l}{} & \multicolumn{6}{c}{\textbf{Traffic Scenarios}}   \\ \cline{1-10}
\multicolumn{2}{c|}{} & \multicolumn{2}{c|}{} & \multicolumn{2}{c|}{Single-Lane Ring} & \multicolumn{2}{c|}{Highway Ramp} & \multicolumn{2}{c}{Signalized Intersection}   \\ \cline{1-10}
\multicolumn{2}{c|}{\textbf{Methods}} & \begin{tabular}[c]{@{}c@{}}Training\\ Complexity\end{tabular} & \multicolumn{1}{c|}{\begin{tabular}[c]{@{}c@{}}The number of \\ source tasks $k$\end{tabular}} & \multicolumn{1}{c}{\begin{tabular}[c]{@{}c@{}}Accel.\\ Guidance\end{tabular}} & \multicolumn{1}{c|}{\begin{tabular}[c]{@{}c@{}}Speed\\ Guidance\end{tabular}} & \multicolumn{1}{c}{\begin{tabular}[c]{@{}c@{}}Accel.\\ Guidance\end{tabular}} & \multicolumn{1}{c|}{\begin{tabular}[c]{@{}c@{}}Speed\\ Guidance\end{tabular}} & \multicolumn{1}{c}{\begin{tabular}[c]{@{}c@{}}Accel.\\ Guidance\end{tabular}} & \multicolumn{1}{c}{\begin{tabular}[c]{@{}c@{}}Speed\\ Guidance\end{tabular}}   \\ \cline{1-10}
\multicolumn{10}{l}{\textbf{Oracle}}  \\ \cline{1-10}
 & \multicolumn{1}{l|}{Oracle Transfer} & $n \times n$ & \multicolumn{1}{c|}{n} & 4.10 & \multicolumn{1}{r|}{4.41} & 5.48 & \multicolumn{1}{r|}{6.30} & 7.71 & 7.71   \\ \cline{2-10}
 & \multicolumn{1}{l|}{Exhaustive RL} & $n$ & \multicolumn{1}{c|}{-} & 3.84 & \multicolumn{1}{r|}{4.29} & 4.24 & \multicolumn{1}{r|}{4.99} & 6.86 & 7.69   \\ \cline{1-10}
\multicolumn{10}{l}{\textbf{Baselines}}   \\ \cline{1-10}
 & \multicolumn{1}{l|}{100\% Unguided} & 0 & \multicolumn{1}{c|}{-} & 3.80 & \multicolumn{1}{r|}{3.80} & 3.95 & \multicolumn{1}{r|}{3.95} & 6.84 & 6.84   \\ \cline{2-10}
 & \multicolumn{1}{l|}{Multitask RL} & $\jhedit{n} \dag$ & \multicolumn{1}{c|}{-} & 3.94 & \multicolumn{1}{r|}{4.28} & \jhedit{4.53} & \multicolumn{1}{r|}{\jhedit{4.42}} & \multicolumn{1}{c}{-} & \multicolumn{1}{c}{-}   \\ \cline{1-10}
\multicolumn{10}{l}{\textbf{Temporal Transfer Learning (Ours)}}  \\ \cline{1-10}
 & \multicolumn{1}{l|}{\multirow{3}{*}{\begin{tabular}[c]{@{}l@{}}Coarse-to-fine Temporal \\ Transfer Learning (CTTL)\end{tabular}}} & $k$ & \multicolumn{1}{c|}{5$\ddag$} & 3.85 & \multicolumn{1}{r|}{4.39} & 4.47 & \multicolumn{1}{r|}{5.31} & \textbf{7.71} & \textbf{7.71}   \\
 & \multicolumn{1}{l|}{} & $k$  & \multicolumn{1}{c|}{10$\ddag$} & 4.02 & \multicolumn{1}{r|}{\textbf{4.40}} & 5.19 & \multicolumn{1}{r|}{5.44} & \textbf{7.71} & \textbf{7.71}   \\
 & \multicolumn{1}{l|}{} & $k$  & \multicolumn{1}{c|}{15$\ddag$} & 3.92 & \multicolumn{1}{r|}{4.39} & \textbf{5.20} & \multicolumn{1}{r|}{5.56} & \textbf{7.71} & \textbf{7.71}   \\ \cline{2-10}
 & \multicolumn{1}{l|}{\multirow{3}{*}{\begin{tabular}[c]{@{}l@{}}Greedy Temporal\\ Transfer Learning (GTTL)\end{tabular}}} & $k$ & \multicolumn{1}{c|}{5} & 3.89 & \multicolumn{1}{r|}{4.39} & 5.19 & \multicolumn{1}{r|}{5.19} & \textbf{7.71} & \textbf{7.71}   \\
 & \multicolumn{1}{l|}{} & $k$ & \multicolumn{1}{c|}{10} & 4.03 & \multicolumn{1}{r|}{\textbf{4.40}} & 5.19 & \multicolumn{1}{r|}{5.54} & \textbf{7.71} & \textbf{7.71}   \\
 & \multicolumn{1}{l|}{} & $k$ & \multicolumn{1}{c|}{15} & \textbf{4.04} & \multicolumn{1}{r|}{\textbf{4.40}} & 5.19 & \multicolumn{1}{r|}{\textbf{6.25}} & \textbf{7.71} & \textbf{7.71}   \\ \cline{1-10}
\multicolumn{10}{l}{\textbf{Ablation}} \\ \cline{1-10}
 & \multicolumn{1}{l|}{\multirow{3}{*}{\begin{tabular}[c]{@{}l@{}}Random Temporal \\ Transfer Learning (RTTL)\end{tabular}}} & $k$ & \multicolumn{1}{c|}{5} & 3.85 & \multicolumn{1}{r|}{4.36} & 4.53 & \multicolumn{1}{r|}{5.70} & 7.69 & \textbf{7.71}   \\
 & \multicolumn{1}{l|}{} & $k$ & \multicolumn{1}{c|}{10} & 3.97 & \multicolumn{1}{r|}{4.39} & 4.89 & \multicolumn{1}{r|}{5.87} & 7.70 & \textbf{7.71}   \\
 & \multicolumn{1}{l|}{} & $k$ & \multicolumn{1}{c|}{15} & 4.01 & \multicolumn{1}{r|}{4.39} & 5.09 & \multicolumn{1}{r|}{6.03} & \textbf{7.71} & \textbf{7.71}   \\ \cline{1-10}
\end{tabular}
\\
\vspace{0.1in}
\small \textdagger: \jhedit{denotes the number of tasks used in multitask RL. It was evaluated after the same number of rollouts of training as other settings.} \\
\small \textdaggerdbl: \jhedit{The performance of the CTTL is assessed after completing the training across a predetermined number of source tasks at budget.}\\
\end{table*}

\jhedit{\Cref{table:simulation-results} and \Cref{fig:three-ttl} illustrate the outcomes of TTL algorithms compared to baselines in different traffic environments. Here, the selected performance metric is the average speed of all vehicles evaluated in all tasks. In scenarios such as the highway ramp, while outflow is the primary reward metric, we also present the average speed as a metric for comparison. This choice is justified by the high correlation between speed and outflow, providing a consistent measure across different scenarios for a more straightforward comparative analysis.}
\jhedit{The bolded values represent the highest performance metrics achieved by TTL algorithms, discounting the Oracle Transfer due to its prohibitive computational demand.}
The TTL algorithms exhibit exemplary performance in both acceleration and speed guidance categories, markedly outperforming the baselines across diverse traffic conditions. 
Remarkably, \jhedit{with few source tasks}, TTL algorithms approach the near-term performance of the oracle transfer. 
Some scenarios require a small number of source tasks to achieve the near-term performance of Oracle transfer, while others demand more extensive iterations. 
Specifically, in the signalized intersection scenario, all Transfer Learning methods yield the top performance when paired with speed guidance. 
\jhedit{This highlights the effectiveness of TTL in optimizing traffic management tasks, particularly when combined with speed guidance.}
\Cref{fig:three-transfered} shows the system-level performance of each task after the temporal transfer learning methods are applied, compared to the exhaustive RL.

The results presented in \Cref{fig:three-ttl} offer an insightful comparison of several training methodologies in the context of coarse-grained advisory autonomy tasks in different traffic scenarios. 
\jhedit{The performance metrics present in \Cref{fig:three-ttl} indicate the average of evaluated performance across the full range of the coarse-grained advisory, and each task is evaluated with the average speed of all vehicles, with higher values denoting better performance. \jhedit{It provides a gauge for the generalizability of source tasks selected from different methods across various target tasks.}}

\textbf{Single-Lane Ring. } 
\Cref{fig:ring-scratch} compares the system performance of acceleration and speed guidance in the single-lane ring road network when trained from scratch. 
When analyzing the results, both guidance types demonstrate excellent overall performance as the guidance hold duration increases, with an average speed increase of approximately 22.22\% for all vehicles in the system. 
However, it is worth noting that the acceleration guidance results were slightly lower than speed guidance.
First, in a single-lane ring environment (\Cref{fig:ttl-ring-acc}), the average speed of GTTL starts higher than both RTTL and CTTL in the first iteration of selecting the source task. 
Despite a slight decrease in the early stages, the speed improves consistently over the iterations and stays competitive against the other methods. 
The performance of GTTL shows that it learns quickly in the initial stages and then continues to optimize its performance in subsequent steps, indicating an effective transfer of knowledge.
It's also worth noting that while no strategy surpasses Oracle Transfer's average speed of $4.10$ m/s, GTTL gets relatively close, reaching final average speeds of approximately $4.04$ m/s. 
While the trends for RTTL are upward as the number of source tasks gets larger, it does not exceed the performance demonstrated by GTTL. 
Multitask RL, although slightly surpassing the baseline, falls short when compared to our GTTL method.

Furthermore, a clear distinction is observed when comparing the number of source tasks required to achieve a given performance level across methods. 
To surpass baselines with exhaustive RL, RTTL necessitates approximately ten steps, whereas GTTL achieves this in merely seven steps, highlighting its efficiency.
These findings strongly advocate the effectiveness of GTTL in such driving scenarios, reinforcing its potential suitability for real-world applications in achieving coarse-grained advisory in mixed autonomy.
Upon examining speed guidance results (\Cref{fig:ttl-ring-vel}), we observe that performance levels are already near-optimal even before applying transfer learning algorithms. 
This observation highlights the intrinsic effectiveness of speed guidance, making the added benefits derived from implementing TTL algorithms less distinguishable in this specific scenario.

\textbf{Highway Ramp. } 
Following the single-lane ring road scenario, we analyze the results from a highway ramp scenario, where the complexity of the traffic situations and interactions are significantly elevated.
\jhedit{Figure~\ref{fig:ramp-scratch} displays the jagged performance of training exhaustively in the highway ramp road network, which could indicate the difficulty of traffic coordination around the ramp and brittleness of RL training. 
However, the overall trend suggests that the average speed of all vehicles decreases as the guidance hold duration increases.}
\jhedit{In both scenarios, the multitask RL approach—trained with access to all target tasks—demonstrates a slight improvement over the unguided baseline but still falls significantly short of the performance offered by the TTL algorithms.}

With the acceleration guidance (\Cref{fig:ttl-ramp-acc}), Oracle Transfer, considered as upper-bound performance, consistently achieved $5.48$ m/s for speed guidance, while the average speed in the unguided case is maintained around $3.95$ m/s. 
CTTL progressively improved the average speed from $4.47$ m/s within the budget of 5 to $5.20$ m/s over 15 source tasks, obtaining the highest performance. 
GTTL started at $5.19$ m/s after the first five steps, which is the highest among other methods, and eventually optimized its performance to $5.19$ m/s across the 15 steps.
Switching to the speed guidance scenario (\Cref{fig:ttl-ramp-vel}), all methods indicated an enhancement compared to the acceleration guidance scenario. 
The RTTL method started at $5.70$ m/s and reached a higher peak speed of $6.03$ m/s. 
Furthermore, the CTTL method increased the average speed from $5.31$ m/s to $5.56$ m/s over the 15 steps. 
GTTL exhibited robustness, initiating at an average speed of $5.19$ m/s and advancing to $6.25$ m/s across the steps.
\jhedit{This result where RTTL outperforms GTTL and CTTL highlights the unpredictable aspects of RL training, indicating that despite GTTL's strategic framework, random task selection by RTTL can, at times, yield comparable or superior results.}

\textbf{Signalized Intersection. }
Analyzing the signalized intersection scenarios with acceleration and speed guidance reveals some notable trends.
When trained exhaustively, the system performance of the signalized intersection remains steady (\Cref{fig:intersection-scratch}). 
\jhedit{During the training of the multitask RL policy, achieving a robust policy that could generalize across all ranges of hold duration tasks turned out to be challenging. Once the policy converged, it unfortunately led to an increase in collisions at intersections.}

For acceleration guidance (\Cref{fig:ttl-inter-acc}), the unguided scenario and exhaustive RL resulted in average speeds of $6.84$ m/s and $6.86$ m/s, respectively, while Oracle Transfer reached $7.71$ m/s. 
TTL methods, particularly CTTL and GTTL, improved significantly, up to the performance of Oracle Transfer.
Speed guidance scenario (\Cref{fig:ttl-inter-vel}) benefits from the transfer learning procedures. 
Both CTTL and GTTL mirrored Oracle Transfer performance, achieving average speeds of around $7.71$ m/s, almost close to the optimal performance. 
\jhedit{We observe that at signalized intersections with speed guidance, even a single source task allows transfer learning methods to match the performance of the oracle. This may be due to the more predictable nature of traffic dynamics at intersections, making them amenable to successful transfer with minimal learning.}
\jhedit{Moreover, the selection of the $K$ parameter for CTTL should be strategically determined by balancing the computational budget and the complexity of the task to optimize the algorithm's performance.}

\subsection{\jhedit{Sensitivity Analysis}} \label{sec:sensitivity-analysis}
\begin{figure}[!t]
    \centering
    \subfloat[Guidance hold duration 2 seconds\label{fig:random-idm-2}]{%
        \includegraphics[width=0.43\textwidth]{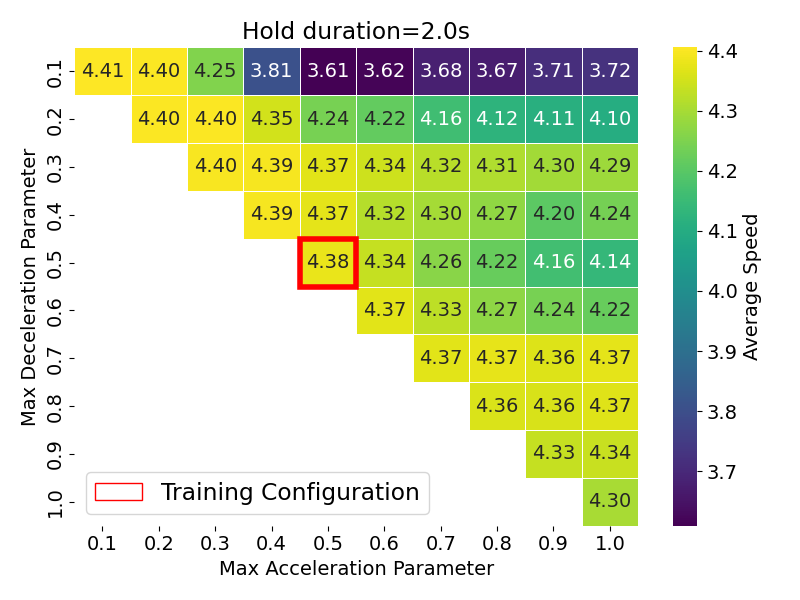}}
    \hfill
    \subfloat[Guidance hold duration 10 seconds\label{fig:random-idm-10}]{%
        \includegraphics[width=0.43\textwidth]{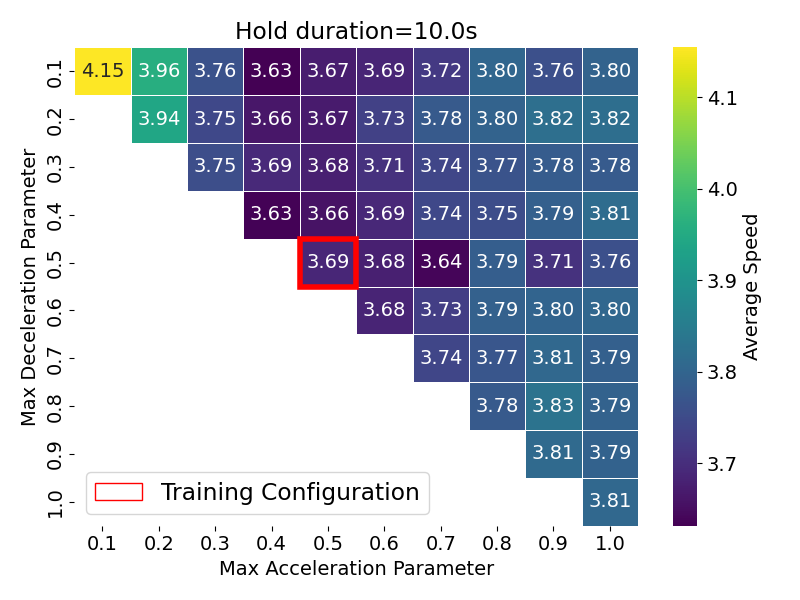}}
    \caption{\textbf{Robustness to Different IDM Parameters.} Heatmaps of average speed (m/s) under different combinations of max-acceleration (x-axis) and max-deceleration (y-axis) parameters. Warmer colors represent higher speeds. Both short (2\,s) and longer (10\,s) hold durations achieve stable performance across diverse driver profiles.}
    \label{fig:random-idm}
\end{figure}

\jhedit{While the IDM is a commonly used car-following model in mixed-autonomy research, it does not capture every nuance of real-world human driving. To evaluate how our learned policies perform under more diverse and less “ideal” driver behaviors, we conducted two complementary sensitivity analyses: (1)~randomizing key IDM parameters, and (2)~inducing abrupt braking events in a subset of human-driven vehicles.}

\jhedit{\textbf{Randomized IDM Parameters.} We first examine how variations in the maximum acceleration and deceleration parameters affect overall system performance. \Cref{fig:random-idm} demonstrate that the learned policies using TTL generalize well beyond a single set of IDM parameters to a broad range of aggressive and conservative driver profiles, both in short and long guidance hold duration.}

\begin{figure}[!t]
    \centering
    \includegraphics[width=0.8\linewidth]{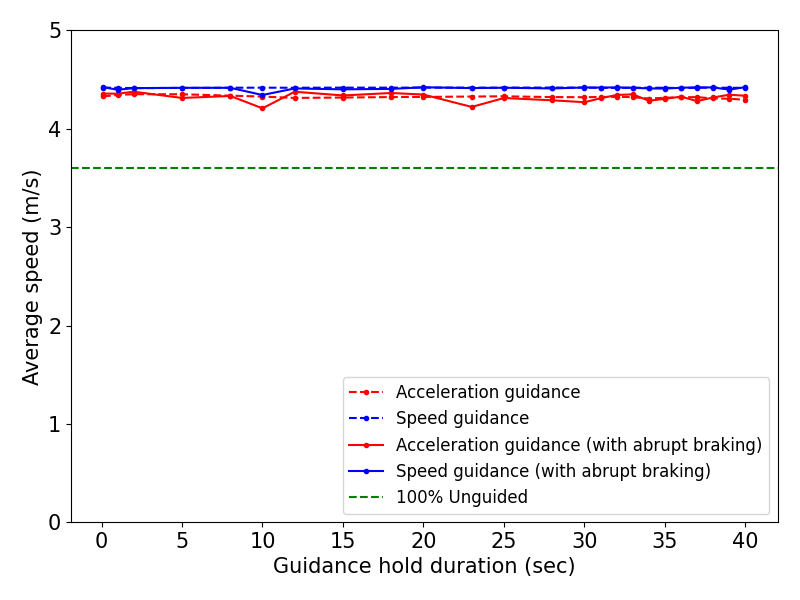}
    \caption{\textbf{Resilience from abrupt braking.} Comparison of our acceleration- and speed-guidance strategies trained with TTL and ``abrupt braking'' in a single-lane ring scenario, along with a 100\% unguided baseline (green dashed line). Coarse-grained guidance maintains a higher average speed even under abrupt driver actions.}
    \label{fig:ring-kick}
\end{figure}

\jhedit{\textbf{Abrupt Braking Events.} In addition to IDM parameter variations, we also tested the resilience of our policy against unexpected driver actions by injecting ``abrupt braking” behaviors. Specifically, a few human-driven vehicles were programmed to perform sudden stops with low probability, disrupting the smooth car-following patterns. \Cref{fig:ring-kick} shows how our acceleration- and speed-guidance policies sustain high average speed in the single-lane ring scenario despite these abrupt disruptions, underscoring the policy’s ability to handle erratic driver maneuvers.}

\jhedit{Overall, these experiments confirm that our learned guidance strategies remain effective under both parameter variations and abrupt disruptions, underscoring their robustness and potential for real-world mixed-autonomy applications.}

\section{Conclusion}
\jhedit{This paper presents temporal transfer learning (TTL) algorithms for coarse-grained advisory, addressing the intrinsic complexity and brittleness of RL algorithms. Through empirical analysis across three traffic scenarios, we evaluate the performance of the advisory system through either acceleration or speed and see how much performance is near-term as instantaneous control. Significant findings show that TTL outperforms baselines in diverse traffic conditions, indicating its potential to improve traffic management in mixed-autonomy environments and reduce computational cost for training RL policies. TTL algorithm is also generic and applicable to different contextual MDP tasks and can find meaningful cross-domain applications especially in industrial automation and robotics. Relaxing our initial assumptions, such as the constant estimation of training performance, could enhance our algorithm's applicability and effectiveness. Future work will explore more complex scenarios and relax the theoretical assumptions to bridge the gap towards real-world applicability. We also underscore the potential into computationally efficient continual learning and fine-tuning strategies for transfer learning, which could offer a viable way for improved traffic management system.
}

\appendices

\section{Proof for Theorem~\ref{theorem:1-step-greedy-for-single-step} \label{appendix:proof-ttl}}

\begin{proof}
\begin{figure}[!t]
    \centering
    \subfloat[$J_k(\delta)$ is symmetric\label{fig:J-symmetric}]{
        \includegraphics[height=3.5cm]{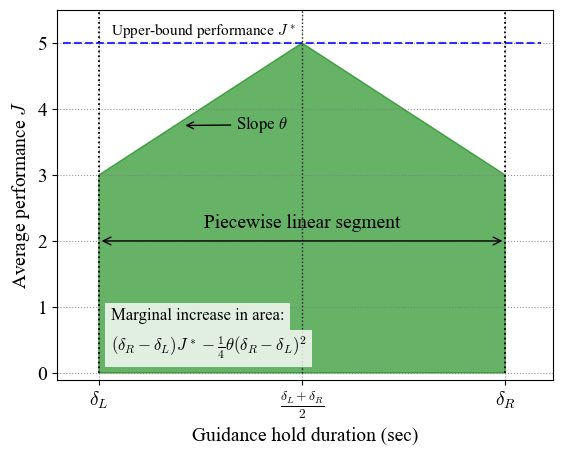}
        \includegraphics[height=3.5cm]{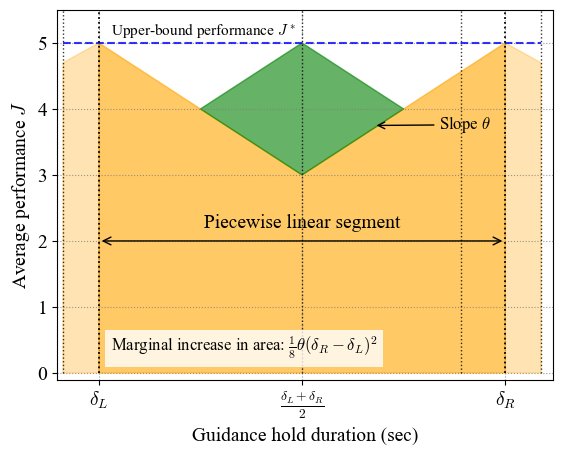}}
    \hfill
    \subfloat[$J_k(\delta)$ is asymmetric\label{fig:J-asymmetric}]{
        \includegraphics[height=3.5cm]{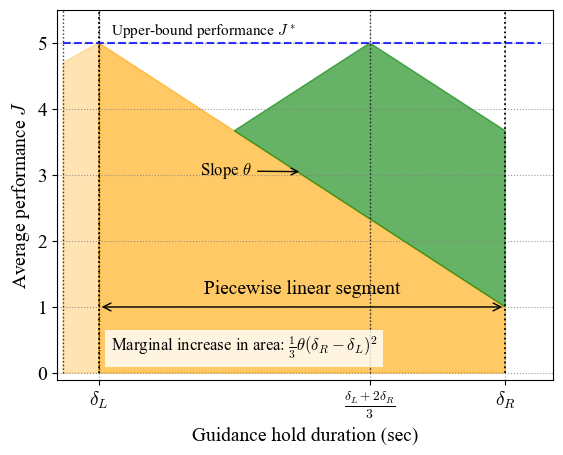}
        \includegraphics[height=3.5cm]{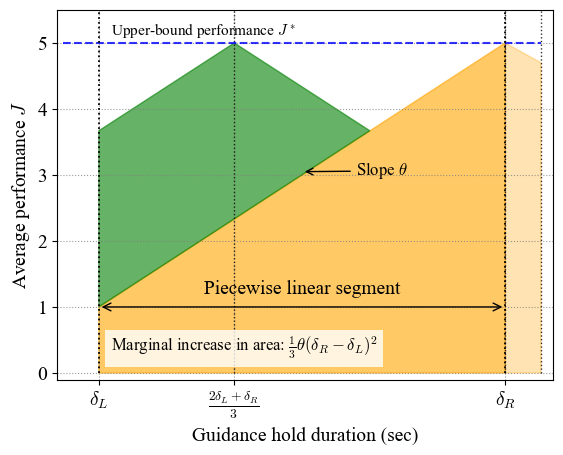}}
    \caption{Decision strategies and corresponding marginal performance increases for symmetric and asymmetric $J_k(\delta)$. (a) When $J_k(\delta)$ is symmetric about the center, the optimal $\delta_k$ is the midpoint, maximizing the area under $J_k$. (b) For asymmetric $J_k(\delta)$, the trisection points yield the optimum depending on the local slope. The green region indicates the marginal gain from the current choice, contrasting with prior coverage (orange).}
    \label{fig:illust-sym-asym}
\end{figure}

\noindent
\Cref{fig:illust-sym-asym} illustrates the one-step decision for selecting the next hold duration $\delta_k$ to maximize the marginal performance gain within a segment $[\delta_L,\delta_R]$. Let $\theta_L,\theta_R>0$ denote the local (piecewise-linear) slopes near the boundaries. We consider three cases.

\paragraph*{Case 1: Symmetric $J_k(\delta)$}
Define the objective as the sum of the two trapezoidal areas created by a split at $\delta'\in[\delta_L,\delta_R]$:
\begin{align*}
    \max_{\delta'} \text{Obj}(\delta')
    &=\tfrac{1}{2}\,(\delta'-\delta_L)\!\left[2J^*-\theta_L(\delta'-\delta_L)\right] \\
    &\quad + \tfrac{1}{2}\,(\delta_R-\delta')\!\left[2J^*-\theta_R(\delta_R-\delta')\right] \\
    &= (\delta_R-\delta_L)J^* - \tfrac{\theta_L}{2}(\delta'-\delta_L)^2 - \tfrac{\theta_R}{2}(\delta_R-\delta')^2 .
\end{align*}
Setting $\frac{\mathrm{d}}{\mathrm{d}\delta'}\text{Obj}(\delta')=0$ gives
\[
    -\theta_L(\delta'-\delta_L)+\theta_R(\delta_R-\delta')=0
    \quad\Rightarrow\quad
    \delta'=\frac{\theta_L\delta_L+\theta_R\delta_R}{\theta_L+\theta_R}.
\]
Under \cref{assume:same-transfer-slope} (\(\theta_L=\theta_R=\theta\)), the optimizer is the midpoint
\(
\delta'=\frac{\delta_L+\delta_R}{2}.
\)
The resulting marginal gain is
\(
\Delta A_k=\frac{1}{8}\,\theta\,(\delta_R-\delta_L)^2.
\)

\paragraph*{Case 2: Positive slope}
For a positively sloped $J_k(\delta)$ across the segment, the objective reduces to
\[
\text{Obj}(\delta')=\tfrac{1}{2}\theta\,(\delta_R+3\delta'-2\delta_L)\,(\delta_R-\delta').
\]
Differentiating and setting to zero yields
\(
\delta'=\tfrac{2\delta_L+\delta_R}{3},
\)
and the marginal area increase is
\(
\Delta A_k=\tfrac{1}{3}\,\theta\,(\delta_R-\delta_L)^2.
\)

\paragraph*{Case 3: Negative slope}
By symmetry, for a negative slope the maximizer is
\(
\delta'=\tfrac{\delta_L+2\delta_R}{3},
\)
with the same marginal gain
\(
\Delta A_k=\tfrac{1}{3}\,\theta\,(\delta_R-\delta_L)^2.
\)

\noindent
These three cases establish the one-step greedy choice.
\end{proof}

\section{Proof for Theorem~\ref{theorem:area-fill-ghost} \label{appendix:proof-area-fill-ghost}}
\begin{proof}
For the initial step,
\[
\tilde{A}_1=(\delta_{\max}-\delta_{\min})J^*-\tfrac{1}{4}\theta(\delta_{\max}-\delta_{\min})^2.
\]
From the geometry in \cref{fig:area-fill-lower}, the next few odd steps satisfy
\begin{align*}
\tilde{A}_3&=(\delta_{\max}-\delta_{\min})J^*-\tfrac{1}{8}\theta(\delta_{\max}-\delta_{\min})^2,\\
\tilde{A}_5&=(\delta_{\max}-\delta_{\min})J^*-\tfrac{1}{16}\theta(\delta_{\max}-\delta_{\min})^2,\\
\tilde{A}_9&=(\delta_{\max}-\delta_{\min})J^*-\tfrac{1}{32}\theta(\delta_{\max}-\delta_{\min})^2, \;\ldots
\end{align*}
In general, for $k=2^i+1$ with $i\in\bN$ and using \cref{assume:upperbound-J}, i.e., $J^*=\theta(\delta_{\max}-\delta_{\min})$,
\begin{align*}
\tilde{A}_{2^i+1}
&=(\delta_{\max}-\delta_{\min})J^*-\tfrac{1}{2^{i+2}}\theta(\delta_{\max}-\delta_{\min})^2 \\
&=\Bigl(1-\tfrac{1}{2^{i+2}}\Bigr)\theta(\delta_{\max}-\delta_{\min})^2.
\end{align*}
To cover at least a fraction $(1-\varepsilon)$ of the full area,
\[
\tilde{A}_{2^i+1}\ge (1-\varepsilon)\,\theta(\delta_{\max}-\delta_{\min})^2
\;\;\Longleftrightarrow\;\;
\varepsilon \ge \tfrac{1}{2^{i+2}}.
\]
Hence,
\[
2^i+1 \;\ge\; \frac{1}{4\varepsilon}+1 \;=\; \frac{4\varepsilon+1}{4\varepsilon},
\]
so at least $\frac{4\varepsilon+1}{4\varepsilon}$ steps are required.
\end{proof}

\section{Proof for Theorem~\ref{lemma:optimality-ctl} \label{appendix:proof-ctl-optimality}}

\begin{proof}
Consider a transfer budget of $K$ and partition $[\delta_{\min},\delta_{\max}]$ into $K{+}1$ subsegments with lengths $l_1,\ldots,l_{K+1}$ (so $\sum_{k=1}^{K+1} l_k=\delta_{\max}-\delta_{\min}$). The remaining area below $J^*$ after the coarse-to-fine temporal transfer (CTTL) sequence equals the sum of triangle areas within each subsegment:
\[
\textstyle
\frac{1}{2}\theta l_1^2,\quad
\frac{1}{4}\theta l_k^2\;(k=2,\ldots,K),\quad
\frac{1}{2}\theta l_{K+1}^2.
\]
Thus we minimize
\begin{align*}
\min_{l_k\ge 0}\quad
& \tfrac{1}{2}\theta l_1^2 + \tfrac{1}{4}\theta l_2^2 + \cdots + \tfrac{1}{4}\theta l_K^2 + \tfrac{1}{2}\theta l_{K+1}^2 \\
\text{s.t.}\quad
& \sum_{k=1}^{K+1} l_k = \delta_{\max}-\delta_{\min}.
\end{align*}
Let
\(
\mathbf{u}=(l_1/\sqrt{2},\,l_2/2,\,\ldots,\,l_K/2,\,l_{K+1}/\sqrt{2})
\)
and
\(
\mathbf{v}=(\sqrt{2},\,2,\,\ldots,\,2,\,\sqrt{2})
\).
Then
\(
\langle \mathbf{u},\mathbf{v}\rangle=\sum_{k=1}^{K+1} l_k=\delta_{\max}-\delta_{\min}
\)
and
\(
\|\mathbf{v}\|^2=4K
\).
By Cauchy–Schwarz,
\[
\theta\|\mathbf{u}\|^2 \;\ge\; \theta \frac{\langle \mathbf{u},\mathbf{v}\rangle^2}{\|\mathbf{v}\|^2}
= \frac{\theta}{4K}(\delta_{\max}-\delta_{\min})^2 .
\]
Equality holds when $\mathbf{u}=\lambda\mathbf{v}$, i.e.,
\(
2l_1=l_2=\cdots=l_K=2l_{K+1}.
\)
Using the sum constraint,
\[
l_1=l_{K+1}=\frac{\delta_{\max}-\delta_{\min}}{2K},
\qquad
l_2=\cdots=l_K=\frac{\delta_{\max}-\delta_{\min}}{K}.
\]
Therefore, the optimal objective value is
\[
A_K^{\text{CTTL}}=\frac{\theta}{4K}(\delta_{\max}-\delta_{\min})^2,
\]
which proves the optimality of the CTTL allocation.
\end{proof}

\section{Proof for Theorem~\ref{theorem:suboptimality-ttl} \label{appendix:proof-ttl-suboptimality}}

\begin{proof}
We consider the cases $K=2^i{+}1$ and $2^{i-1}{+}1<K<2^i{+}1$ with $i\in\bN$.

\paragraph*{Case 1: $K=2^i{+}1$}
From Appendix~\ref{appendix:proof-area-fill-ghost},
\begin{align*}
    A^{\text{CTTL}}_K&= \Bigl(1-\tfrac{1}{4K}\Bigr)\theta(\delta_{\max}-\delta_{\min})^2,\\
    A^{\text{GTTL}}_K &\ge \tilde{A}^{\text{GTTL}}_K = \Bigl(1-\tfrac{1}{4(K-1)}\Bigr)\theta(\delta_{\max}-\delta_{\min})^2.
\end{align*}
Hence,
\[
A^{\text{CTTL}}_K - A^{\text{GTTL}}_K
\le \frac{1}{4K(K-1)}\,\theta(\delta_{\max}-\delta_{\min})^2 .
\]

\paragraph*{Case 2: $2^{i-1}{+}1<K<2^i{+}1$}
Using
\(
\tilde{A}^{\text{GTTL}}_{2^i+1}=(1-\tfrac{1}{2^{i+2}})\theta(\cdot)^2
\)
and
\(
\tilde{A}^{\text{GTTL}}_{2^{i-1}+1}=(1-\tfrac{1}{2^{i+1}})\theta(\cdot)^2
\),
the uniform increments over this range are
\[
\tilde{A}^{\text{GTTL}}_{K+1}-\tilde{A}^{\text{GTTL}}_K
= \frac{1}{2^{2i+1}}\,\theta(\delta_{\max}-\delta_{\min})^2.
\]
Therefore,
\begin{align*}
\tilde{A}^{\text{GTTL}}_K
&= \tilde{A}^{\text{GTTL}}_{2^{i-1}+1}
+ \frac{K-(2^{i-1}+1)}{2^{2i+1}}\theta(\delta_{\max}-\delta_{\min})^2 \\
&= \Bigl(1 + \frac{K-3\cdot 2^{i-1}-1}{2^{2i+1}}\Bigr)\theta(\delta_{\max}-\delta_{\min})^2.
\end{align*}
Then
\begin{align*}
A^{\text{CTTL}}_K - A^{\text{GTTL}}_K
&\le A^{\text{CTTL}}_K - \tilde{A}^{\text{GTTL}}_K\\
&\le \frac{1}{2^{2i+1}}\,\theta(\delta_{\max}-\delta_{\min})^2\\
&\le \frac{1}{2(K-1)^2}\,\theta(\delta_{\max}-\delta_{\min})^2,
\end{align*}
which yields the claimed bound.
\end{proof}

\section{Experimental details for modular road network \label{appendix:exp-modular-road-network}}
\noindent In mixed autonomy roadway settings, we investigate the traffic scenarios covered in the previous works \cite{yan_reinforcement_2021, wu_flow_2022, yan_unified_2022}, including the following road networks: single-lane ring, highway ramp, and signalized intersection. 

\paragraph{Single-lane Ring}
Inspired by Sugiyama \cite{sugiyama_traffic_2008}, the circular ring has a circumference of \(250\)\,m.
The single-lane ring environment aims to increase the average velocity of all vehicles in the road network.

The reward function is the average speed of all vehicles.
\begin{equation}
    r(s,a) = \frac{1}{n}\sum_{\forall i} v_i(s,a)
    \label{eqn:reward_ring}
\end{equation}

\paragraph{Highway Ramp}
The objective in the highway ramp environment was to increase the outflow given the same inflow. 
\jhedit{The reward function in the highway ramp scenario focuses on traffic flow efficiency. It is defined as the number of vehicles exiting the system.}
\jhedit{Specifically, we compare the average speed of all vehicles in the system as a performance measure.}

\paragraph{Signalized Intersection}
We have designed a single-lane, 4-way signalized intersection regulated by a static traffic signal phase. A multi-tasking training approach is employed to train this intersection. Specifically, we use a multi-task reinforcement learning (RL) strategy, considering various penetration rates to simulate different levels of human-guided vehicle presence. Nonetheless, when evaluating the effectiveness of this strategy, we focus on scenarios with a 0.1 penetration rate. This allows us to assess the performance of the trained RL policy in conditions where only 10\% of the vehicles are controlled by the RL policy, and the remaining 90\% operate under human \jhedit{driving model}.
\jhedit{
    For the signalized intersection, the reward function incorporates multiple components to optimize various aspects of traffic flow:

    \begin{equation}
    r(s,a) = \frac{1}{n}\sum_{\forall i} v_i(s,a) - \alpha_1 \cdot P_{\text{stop}} - \alpha_2 \cdot A_{\text{accel}} - \alpha_3 \cdot F_{\text{consumption}}
    \label{eqn:reward_inter}
    \end{equation}
    
    where:
    \begin{itemize}
        \item \( P_{\text{stop}} \) is a penalty for vehicles stopped or moving slower than a set threshold (typically 1 m/s), aimed at reducing delays.
        \item \( A_{\text{accel}} \) penalizes abrupt acceleration or deceleration to encourage smoother driving.
        \item \( F_{\text{consumption}} \) penalizes high fuel consumption, promoting eco-friendly driving practices.
    \end{itemize}

    The weights \( \alpha_1, \alpha_2, \alpha_3 \) are hyperparameters that fine-tune the significance of each penalty within the reward function, allowing for a balanced consideration of speed, safety, and efficiency based on the specific goals of each scenario.
    }
\jhedit{Similar to the highway ramp scenario, we compare the average speed of all vehicles in the system as a performance measure.}

\jhedit{\Cref{table:experiment-setup} provides the detailed experimental setup for RL, microscopic traffic simulation, car-following models, and the traffic scenarios used in the experiments.}

\begin{table*}[!t]
\centering
\footnotesize
\caption{\jhedit{Experimental Parameters for Reinforcement Learning, Temporal Transfer Learning, Simulations, Car following models, and Scenarios.}}
\label{table:experiment-setup}
\begin{tabular}{ l|l|l } 
 \hline
 \textbf{Type}     &\textbf{Experiment parameters}     & \textbf{Value}   \\
 \hline
 Reinforcement&Training epochs         & 1,000       \\ 
 Learning&Discount factor         & 0.999       \\ 
 &Test epochs         & 50       \\ 
 &Number of discrete action space             & 10             \\
 \hline
 Simulation&Simulation step         & 0.1 s/step       \\ 
 &Warmup steps             & 500 sec             \\
 &Timestep horizon             & 1,000 sec             \\
 \hline
 \jhedit{Car Following Model}&\jhedit{Model}         & \jhedit{Intelligent Driver Model \cite{treiber_congested_2000}}       \\ 
 &\jhedit{Maximum acceleration}             & \jhedit{1 {m/s$^2$}}             \\
 &\jhedit{Comfortable deceleration}              & \jhedit{1.5 {m/s$^2$}}             \\
 &\jhedit{Desired velocity}              & \jhedit{30 {m/s}}             \\
 &\jhedit{Minimum spacing}              & \jhedit{2 {m}}             \\
 &\jhedit{Desired time headway}              & \jhedit{1 {sec}}             \\
 &\jhedit{Exponent}              & \jhedit{4}             \\
 \hline
 Ring&Circumference         & 250 {m}       \\ 
 &Number of controlled vehicles         & 1       \\ 
 &Total number of vehicles         & 22       \\ 
 &Speed limit         & 10 {m/s}       \\ 
 \hline
 Highway ramp&Mainlane inflow rate         & 2,000 {veh/hr}       \\ 
 &On-ramp inflow rate         & 300 {veh/hr}       \\ 
 &Guided vehicle penetration rate         & 0.1       \\ 
 &Speed limit         & 30 {m/s}       \\ 
 \hline
 Signalized intersection&Inflow rate         & 400 {veh/hr}       \\ 
 &Signal phase time (green, red)        & 35, 45 {sec}       \\ 
 &Guided vehicle penetration rate         & 0.1       \\ 
 &Speed limit         & 14 {m/s}       \\ 
 &\jhedit{Weight for stop penalty}         & \jhedit{35}       \\ 
 &\jhedit{Weight for acceleration}         & \jhedit{1}       \\ 
 &\jhedit{Weight for fuel consumption}         & \jhedit{1}       \\ 
 \hline
 Temporal Transfer Learning&Transfer budget         & 15       \\ 
 \hline
\end{tabular}
\end{table*}

\section*{Acknowledgments}
\noindent The authors acknowledge the MIT SuperCloud and Lincoln Laboratory Supercomputing Center for providing HPC resources that have contributed to the research results reported within this paper. 
This work was partially supported by the MIT Energy Initiative (MITEI) Mobility Systems Center, the Kwanjeong scholarship, the National Science Foundation (NSF) under grant number 2149548, the NSF CAREER award (2239566), and the MIT Amazon Science Hub.
The authors would like to thank Katie Driggs-Campbell for the insightful discussion about advisory autonomy.

\bibliographystyle{IEEEtran}
\bibliography{reference}

\begin{IEEEbiography}[{\includegraphics[width=1in, height=1.25in, clip, keepaspectratio]{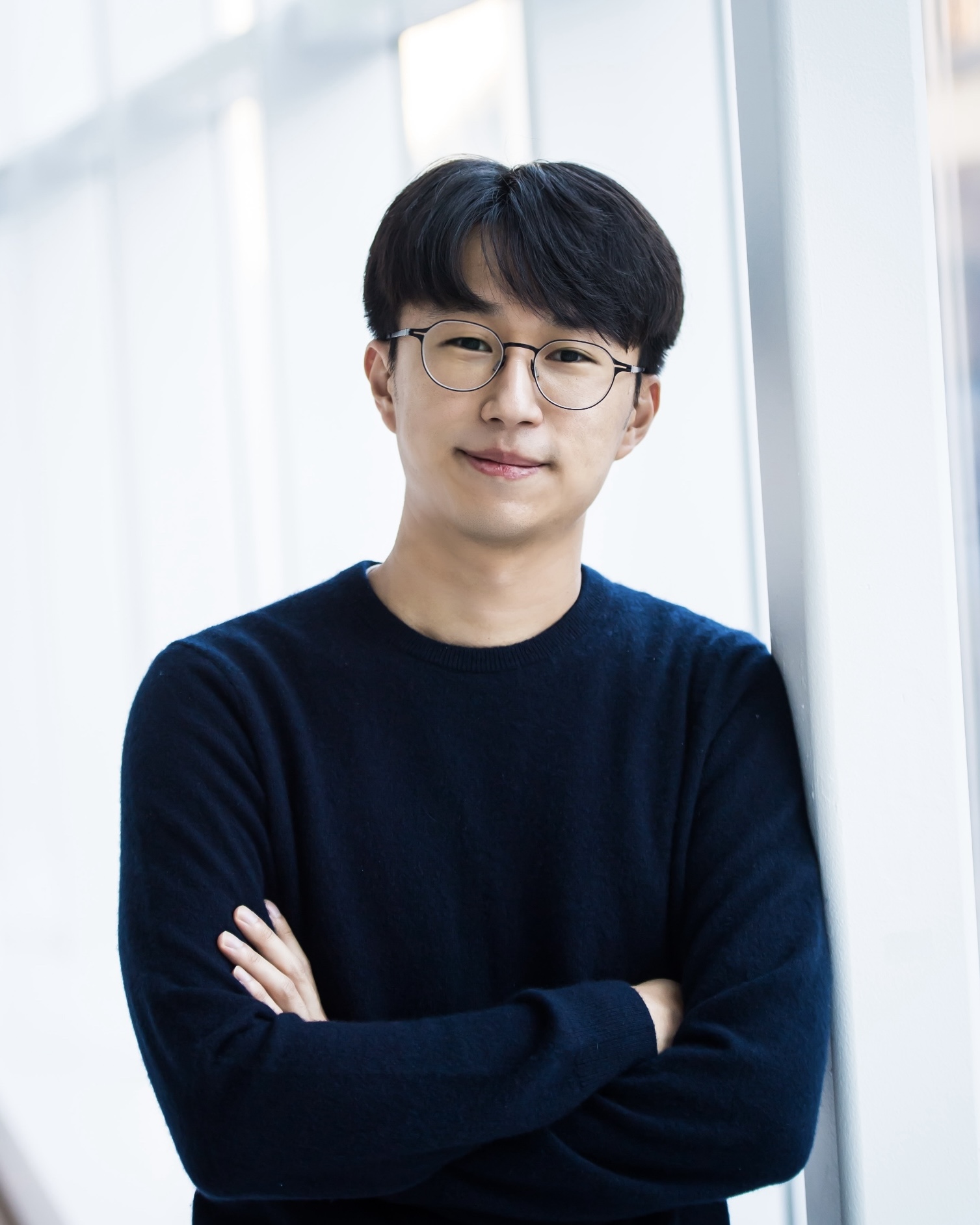}}]{Jung-Hoon Cho}
(Graduate Student Member, IEEE)
received the B.S. and M.S. degrees in CEE from Seoul National University. He is currently pursuing the Ph.D. degree with the Department of Civil and Environmental Engineering and the Laboratory for Information and Decision Systems, Massachusetts Institute of Technology (MIT). His research interests include the convergence of transportation systems and machine learning.
This involves delving into smart infrastructure, such as autonomous driving systems, traffic control in scenarios of mixed autonomy, and urban data intelligence through the application of machine learning.
\end{IEEEbiography}

\begin{IEEEbiography}[{\includegraphics[width=1in, height=1.25in, clip, keepaspectratio]{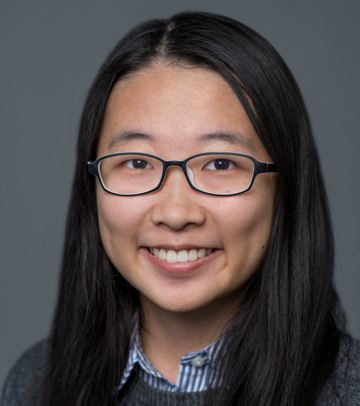}}]{Sirui Li} (Member, IEEE) received the B.S. degree in computer science and mathematics from Washington University, St. Louis, MO, USA, in 2019, and the Ph.D. degree in Social and Engineering System and Statistics from Massachusetts Institute of Technology, Cambridge, MA, USA, in 2025. Her research interests include areas of machine learning for combinatorial optimization and control analysis for transportation systems.
\end{IEEEbiography}

\begin{IEEEbiography}[{\includegraphics[width=1in, height=1.25in, clip, keepaspectratio]{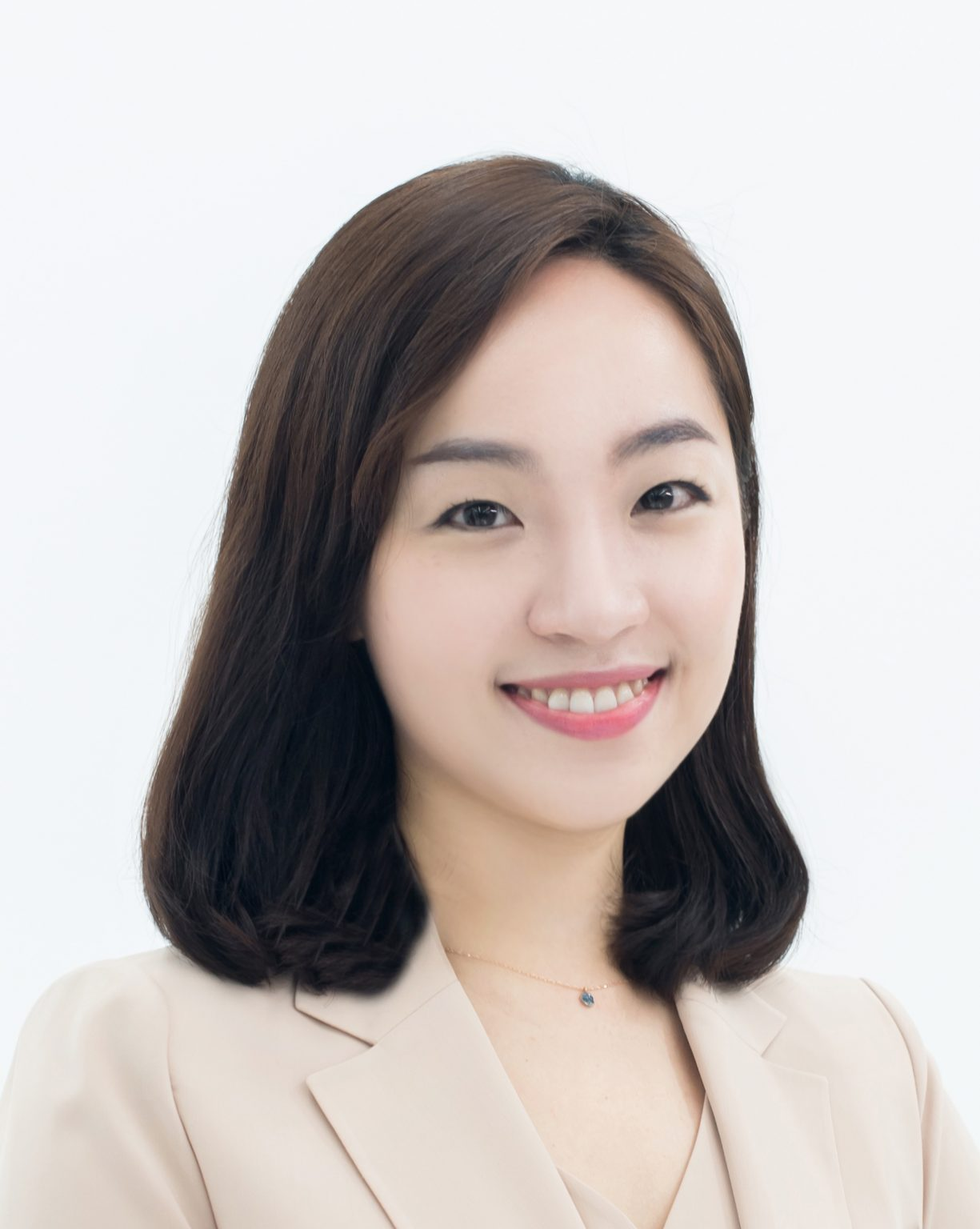}}]{Jeongyun Kim} (Member, IEEE) received the B.S., M.S., and Ph.D. degrees from the Department of Civil and Environmental Engineering, Korea Advanced Institute of Science and Technology (KAIST), in 2014, 2016, and 2021, respectively. She was a Post-Doctoral Researcher with Massachusetts Institute of Technology (MIT). She is currently an Assistant Professor with the Department of Mechanical and Automotive
Engineering, Seoul National University of Science and Technology. Her research is in the area of modeling urban mobility and developing solutions to the challenges related to the autonomous urban mobility system. Her current research interests include urban mobility analytics using large-scale movement data and modeling autonomy control system using machine learning.
\end{IEEEbiography}

\begin{IEEEbiography}[{\includegraphics[width=1in, height=1.25in, clip, keepaspectratio]{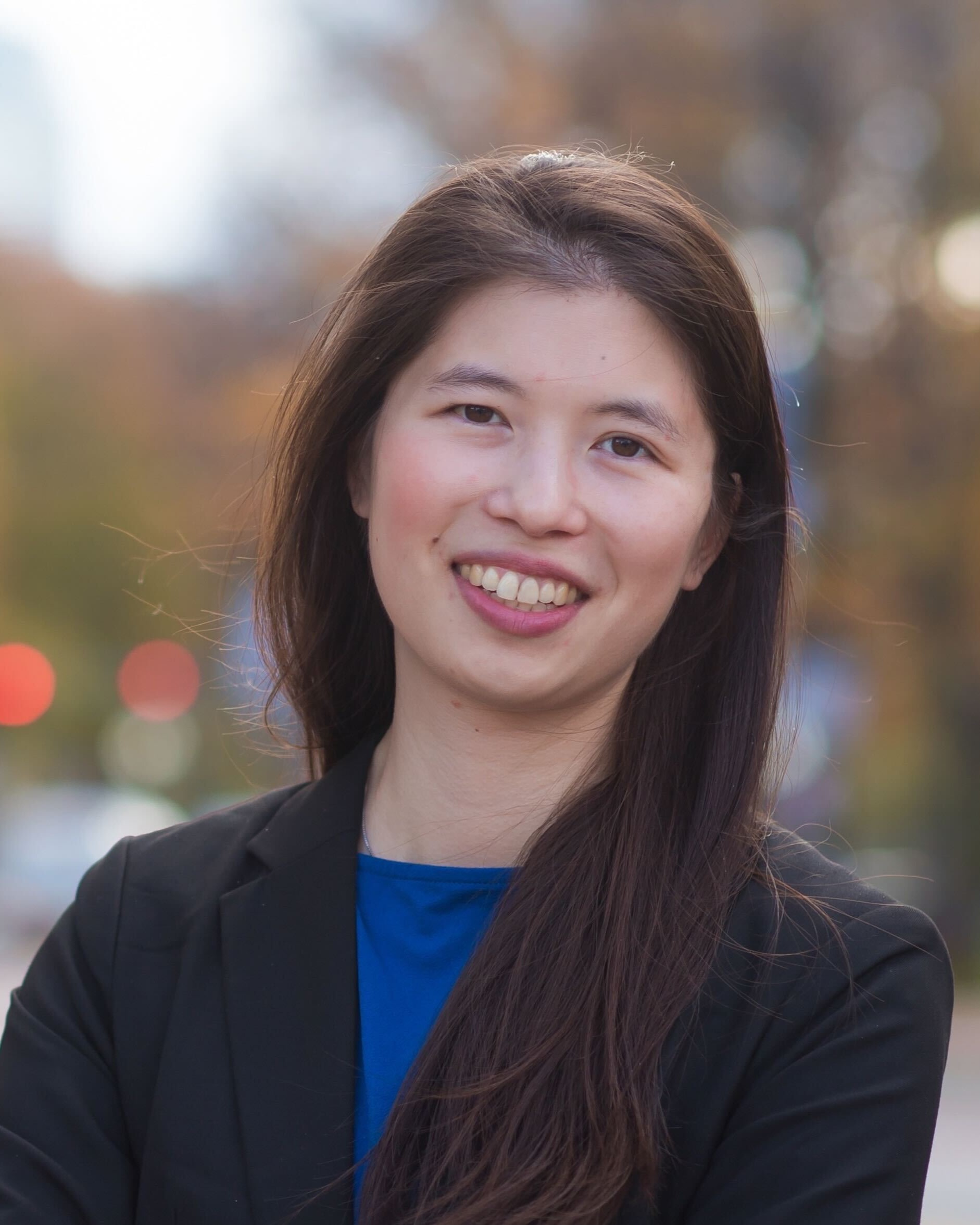}}]{Cathy Wu}
(Member, IEEE) 
received B.S. and M.Eng. degrees from MIT and a Ph.D. degree from UC Berkeley, all in EECS. She was a post-doctoral researcher at Microsoft Research. She is currently the Class of 1954 Career Development Associate Professor at MIT in LIDS, CEE, and IDSS. Her research centers around AI for Engineering, particularly to accelerate R\&D for the public interest and in civil infrastructure. In recent years, her focus has been on using machine learning to tackle bottlenecks in solving control and optimization problems in mobility. Cathy is the recipient of the NSF CAREER, PhD dissertation awards, and the Ole Madsen Mentoring Award. She serves on the Board of Governors for the IEEE ITSS, is a Program Co-chair for RLC 2025, and is an Associate Editor (or equivalent) for ICML, NeurIPS, ICRA, and TR Part C. She is also the inaugural Chair and Co-founder of the REproducible Research In Transportation Engineering (RERITE) Working Group.
\end{IEEEbiography}

\end{document}